\documentclass[nohyperref]{article}

\usepackage{microtype}
\usepackage{graphicx}
\usepackage{subcaption}
\usepackage{booktabs} %
\usepackage{multirow}
\usepackage{hyperref}
\usepackage{tikz}
\usepackage{pgfplots}
\DeclareUnicodeCharacter{2212}{−}
\usepgfplotslibrary{groupplots,dateplot}
\usetikzlibrary{patterns,shapes.arrows}
\pgfplotsset{compat=newest}

\hypersetup{
    colorlinks,
    linkcolor={red!60!black},
    citecolor={blue!60!black},
    urlcolor={blue!60!black}
}
\usepackage[accepted]{icml2022}

\usepackage{amsmath}
\usepackage{amssymb}
\usepackage{mathtools}
\usepackage{amsthm}
\usepackage{enumitem}
\usepackage[capitalize,noabbrev]{cleveref}

\usepackage{amsmath,amsfonts,bm,amssymb,mathtools,amsthm}
\usepackage{color,xcolor,xspace}
\usepackage{booktabs}
\usepackage{thm-restate}
\usepackage{bbding}
\usepackage{pifont}
\usepackage{wasysym}

\newcommand{\bb}[1]{{\mathbb{#1}}}
\newcommand{\norm}[1]{{\lVert {#1} \rVert}}
\newcommand{\diff}{\mathrm{d}}

\def\eqref#1{Eq.(\ref{#1})}

\def\1{\bm{1}}

\def\rvx{{\mathbf{x}}}
\def\rvy{{\mathbf{y}}}
\def\rvz{{\mathbf{z}}}

\DeclareMathAlphabet{\mathsfit}{\encodingdefault}{\sfdefault}{m}{sl}
\SetMathAlphabet{\mathsfit}{bold}{\encodingdefault}{\sfdefault}{bx}{n}

\def\gD{{\mathcal{D}}}

\def\gL{{\mathcal{L}}}
\def\gM{{\mathcal{M}}}
\def\gN{{\mathcal{N}}}
\def\gO{{\mathcal{O}}}

\def\gT{{\mathcal{T}}}

\def\gX{{\mathcal{X}}}

\newcommand{\R}{\mathbb{R}}

\DeclareMathOperator*{\argmax}{arg\,max}

\DeclareMathOperator*{\plim}{plim}

\theoremstyle{plain}
\newtheorem{theorem}{Theorem}[section]
\newtheorem{proposition}[theorem]{Proposition}

\newtheorem{corollary}[theorem]{Corollary}
\theoremstyle{definition}
\newtheorem{definition}[theorem]{Definition}
\newtheorem{assumption}[theorem]{Assumption}
\theoremstyle{remark}
\newtheorem{remark}[theorem]{Remark}
\newtheorem{claim}[theorem]{Claim}
\usepackage[textsize=tiny]{todonotes}

\newcommand{\methodshort}{LFBO}

\icmltitlerunning{A General Recipe for Likelihood-free Bayesian Optimization}

\begin{document}

\twocolumn[
\icmltitle{A General Recipe for Likelihood-free Bayesian Optimization}

\icmlsetsymbol{equal}{*}

\begin{icmlauthorlist}
\icmlauthor{Jiaming Song}{equal,nv}
\icmlauthor{Lantao Yu}{equal,st}
\icmlauthor{Willie Neiswanger}{st}
\icmlauthor{Stefano Ermon}{st}
\end{icmlauthorlist}

\icmlaffiliation{st}{Stanford University}
\icmlaffiliation{nv}{NVIDIA (Work done while at Stanford).}

\icmlcorrespondingauthor{Jiaming Song}{jiamings@nvidia.com}

\icmlkeywords{Machine Learning, ICML}

\vskip 0.3in
]

\printAffiliationsAndNotice{
    \vspace{-4pt}\\
    \textbf{Project website}: \url{https://lfbo-ml.github.io/} \\
    \textbf{Code}: \url{https://github.com/lfbo-ml/lfbo}
    \vspace{-2pt}\\ \; \\
    \icmlEqualContribution
} %

\begin{abstract}
The acquisition function, a critical component in Bayesian optimization (BO), can often be written as the expectation of a utility function under a surrogate model. However, to ensure that acquisition functions are tractable to optimize, restrictions must be placed on the surrogate model and utility function. To extend BO to a broader class of models and utilities, we propose likelihood-free BO (LFBO), an approach based on likelihood-free inference. LFBO directly models the acquisition function without having to separately perform inference with a probabilistic surrogate model. 
We show that computing the acquisition function in LFBO can be reduced to optimizing a weighted classification problem, where the weights correspond to the utility being chosen.
By choosing the utility function for expected improvement (EI), LFBO outperforms various state-of-the-art black-box optimization methods on several real-world optimization problems.
LFBO can also effectively
leverage composite structures of the objective function, which further improves its regret by several orders of magnitude.

\end{abstract}

\section{Introduction}

Bayesian optimization (BO) is a framework for global optimization of black-box functions, whose evaluations are expensive. %
Originated from applications on engineering design~\citep{kushner1962versatile,kushner1964new,movckus1975bayesian}, BO has seen successes in various domains, including automated machine learning~\citep{bergstra2011algorithms,snoek2012practical,swersky2013multi}, simulation optimization~\citep{pearce2017bayesian,wang2020nonparametric}, drug discovery~\citep{griffiths2020constrained}, graphics~\citep{brochu2007active,koyama2017sequential}, robotics~\citep{calandra2016bayesian}, and battery charging protocols design~\citep{attia2020closed}. %

A cornerstone of BO methods is the acquisition function that determines %
which candidates to propose for function evaluations. Many acquisition functions can be written in the form of the expectation of a utility function over a surrogate model~\citep{wilson2018maximizing,garnett_bayesoptbook_2022}. 
To evaluate and optimize these acquisition functions effectively, the surrogate model is often chosen to produce tractable probability estimates (\textit{e.g.}, Gaussian processes~\citep{rasmussen2003gaussian}), and the utility functions are often chosen such that the expectations can be computed analytically, such as Probability of Improvement (PI,~\citep{kushner1964new}).

However, the requirements of tractability and efficiency often exclude popular families of surrogate models and utility functions relevant to the optimization problem.  
For instance, implicit models defined via simulators~\citep{diggle1984monte} may fit the data-generating process better but do not have tractable probability estimates. Meanwhile, objectives that are defined over many correlated outcomes are often partially known (\textit{i.e.}, in a grey-box BO setting~\citep{astudillo2022thinking,maddox2021bayesian}), yet utilities defined for these objective functions may not have tractable integrals even for Gaussian process surrogate models. 
While Monte Carlo methods can be employed here~\citep{astudillo2019bayesian,kleinegesse2019efficient}, they scale poorly when many samples are needed for low variance estimates. %

Inspired by how likelihood-free density ratio estimation extends variational inference to implicit models~\citep{sugiyama2012density,tran2017hierarchical}, we describe a general approach to obtain ``likelihood-free'' acquisition functions that extend BO to a broader class of models and utilities.
We develop a likelihood-free method that directly estimates the ratio between any pair of non-negative measures; this includes any acquisition function with a non-negative utility function, and allows us to directly model it without having to separately perform inference with a probabilistic surrogate model. Our likelihood-free BO (\methodshort{}) approach can be reduced to a weighted %
classification problem, which is easy to implement for general utility functions.

We evaluate our \methodshort{} method empirically on a number of synthetic and real optimization problems. %
On several real-world benchmarks, LFBO %
outperforms various state-of-the-art methods in black-box optimization. On two neural network tuning tasks, we show that LFBO can further be improved with other novel utility functions. On an environmental modeling problem~\citep{astudillo2019bayesian}, LFBO can leverage the special composite structure of the objective function and improve its optimization efficiency.

\begin{figure}
    \centering
    \includegraphics[width=0.49\textwidth]{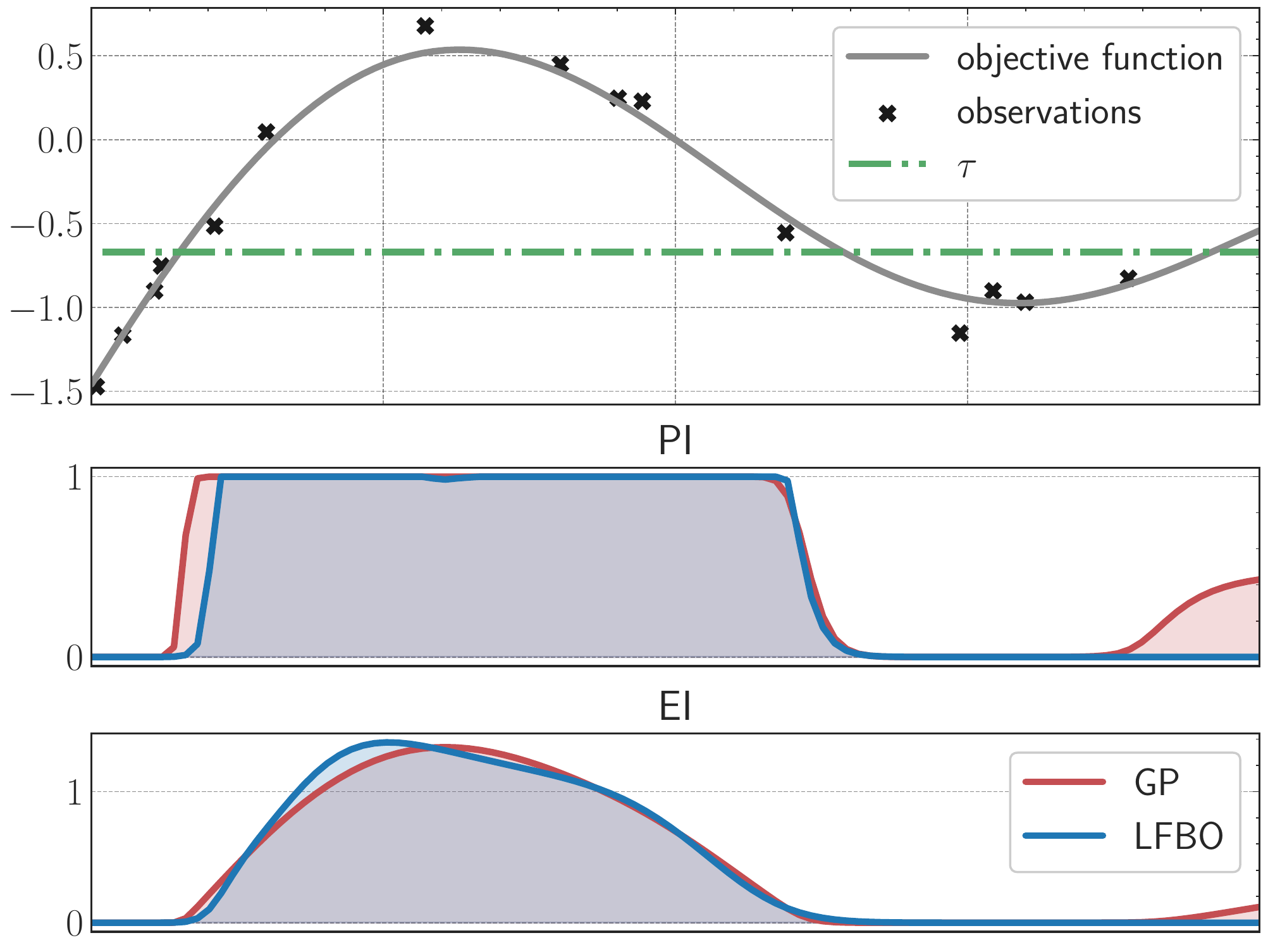}
    \caption{BO over a synthetic function $g(x) = - \sin(3x) - x^2 + 0.6x$ with noise $\gN(0, 0.01)$. (\textit{Top}) the objective function $g(x)$ with $n = 15$ noisy observations and the threshold $\tau$. (\textit{Bottom}) the evaluated acquisition function values with Gaussian Process (GP, \textit{red}) and our likelihood-free BO (LFBO, \textit{blue}) over PI and EI utility functions. Our LFBO acquisition function does not separately perform probabilistic inference over a surrogate model and can adapt to any non-negative utility function. }
    \label{fig:demo}
\end{figure}

\section{Background}
Bayesian optimization (BO) aims to find a strategy that effectively maximizes a black-box function $g: \gX \to \bb{R}$, given available information regarding $g$ as a set of $n$ observations $\gD_n = \{(\rvx_i, y_i)\}_{i=1}^{n}$; it is commonly assumed that $y_i = g(\rvx_i) + \epsilon$ is observed with noise $\epsilon$. In sequential decision making settings, at each iteration $i$, BO methods select the next query $\rvx_i$ via maximizing an acquisition function $\gL: \gX \to \R$, which utilizes a surrogate model $\gM$. %

$\gM$ may provide a probabilistic interpretation of $g$, which can then be used to 
evaluate the acquisition function at any point $\rvx \in \gX$. After updating the model parameters based on $\gD_n$, a belief over $g(\rvx)$ can be formed as $p(y | \rvx, \gD_n)$. 
The Gaussian process (GP,~\citet{rasmussen2003gaussian}) is a popular surrogate model due to its posterior being analytically tractable; however, the basic form of GP requires $\gO(n^3)$ computational complexity to perform posterior inference, making it difficult to scale with large observation sets. Although researchers have proposed sparse GPs with approximations for improved scalability, the inference time still scales with the number of data points (e.g., $\mathcal{O}(n)$ in \citep{titsias2009variational}).

\newcommand{\PI}{\mathrm{PI}}
\newcommand{\EI}{\mathrm{EI}}
\newcommand{\DR}{\mathrm{DR}}

\paragraph{Acquisition functions based on expected utility.}
Many acquisition functions at an input location $\rvx$ are defined as an expected utility (EU) over the posterior belief $p(y | \rvx, \gD_n)$:
\begin{align}
    \gL^{(u)}(\rvx; \gD_n, \tau) & = \bb{E}_{y \sim p(y | \rvx, \gD_n)}[u(y; \tau)] \\
    & = \int u(y; \tau) p(y | \rvx, \gD_n) \diff y, \label{eq:myopic-aqf-def}
\end{align}
where $u(y; \tau)$ is a chosen utility function with hyperparameter $\tau \in \gT$ ($\gT$ is the set of allowed values) %
that specifies the utility of observing $y$ at $\rvx$ and controls the exploration-exploitation trade-off~\citep{wilson2018maximizing,garnett_bayesoptbook_2022}.
The most common examples are Probability of Improvement (PI,~\citet{kushner1964new}) whose utility indicates whether $y$ exceeds $\tau$: 
\begin{align}
    u^{\PI}(y; \tau) := \bb{I}(y - \tau > 0),
\end{align}
where $\bb{I}$ is the binary indicator function;
and Expected Improvement (EI,~\citet{mockus1978application}) whose utility indicates how much $y$ exceeds the threshold $\tau$:
\begin{align}
    u^{\EI}(y; \tau) := \max(y - \tau, 0).
\end{align}
Other examples include Entropy Search (ES,~\citet{hennig2012entropy,hernandez2014predictive}) and Knowledge Gradient~\citep{frazier2008knowledge,scott2011correlated}, whose utilities are more complicated expectations themselves.  %

\paragraph{Acquisition functions based on density ratios.}
While most methods used in BO model the distribution of outcomes of the black-box function $g$, 
some methods do not use an explicit model in the form of $p(y | \rvx, \gD_n)$, and instead model the acquisition function directly via density ratios over $\rvx$. 
Density ratio (DR) acquisition functions~\citep{bergstra2011algorithms} expresses the acquisition function with the ratio of two model densities\footnote{We overload the term ``density'' to include ``probability'' for discrete measures; we also consider maximization instead of minimization of objectives.} over $\rvx$, $p(\rvx | y > \tau, \gD_n)$ and $p(\rvx | y \leq \tau, \gD_n)$, that models the densities of $\rvx$ conditioned on $y$ being above or below the threshold $\tau$, respectively\footnote{Throughout the paper, we use the notation $\gL$ to indicate acquisition functions that explicitly use tractable forms of $p(y | \rvx, \gD_n)$, and the notation $L$ to indicate acquisition functions that do not (\textit{e.g.}, acquisition functions based on density ratios).}: 
\begin{align}
    & L^{\text{DR}}(\rvx; \gD_n, \tau)   = \frac{p(\rvx | y > \tau, \gD_n)}{p(\rvx | \gD_n)} \label{eq:dr-acqf-def} \\
    = & \ \frac{p(\rvx | y > \tau, \gD_n)}{p(\rvx | y > \tau, \gD_n) \gamma_{\gD_n} + p(\rvx | y \leq \tau, {\gD_n}) (1 - \gamma_{\gD_n})}, \nonumber
\end{align}
where $\gamma_{\gD_n} := p(y > \tau | {\gD_n})$. Intuitively, when the ratio in \autoref{eq:dr-acqf-def} is high for a certain $\rvx$, it is believed to produce favorable outcomes (larger than $\tau$) with a high probability; this makes density ratios being successfully used in hyperparameter tuning and neural architecture search~\citep{bergstra2013making}.

To estimate this density ratio, \citeauthor{bergstra2011algorithms} estimate the numerator and denominator with two seprate Tree Parzen Estimators (TPE), whereas
Bayesian optimization by density-ratio estimation (BORE,~\citet{tiao2021bore}) estimate this density ratio directly with likelihood-free inference~\citep{sugiyama2012density}.
  \citeauthor{bergstra2011algorithms} and \citeauthor{tiao2021bore} further argue that the DR acquisition function is equivalent to the EI acquisition function up to some constant factor; thus DR is limited to one particular type of utility function (which is supposedly EI). In Section~\ref{sec:method}, however, we prove that DR is equivalent to PI instead of EI, and then propose a solution to generalize DR-based BO to other acquisition functions, including EI.

\section{Likelihood-free Bayesian Optimization}
\label{sec:method}

\begin{figure*}
    \centering
    \includegraphics[width=\textwidth]{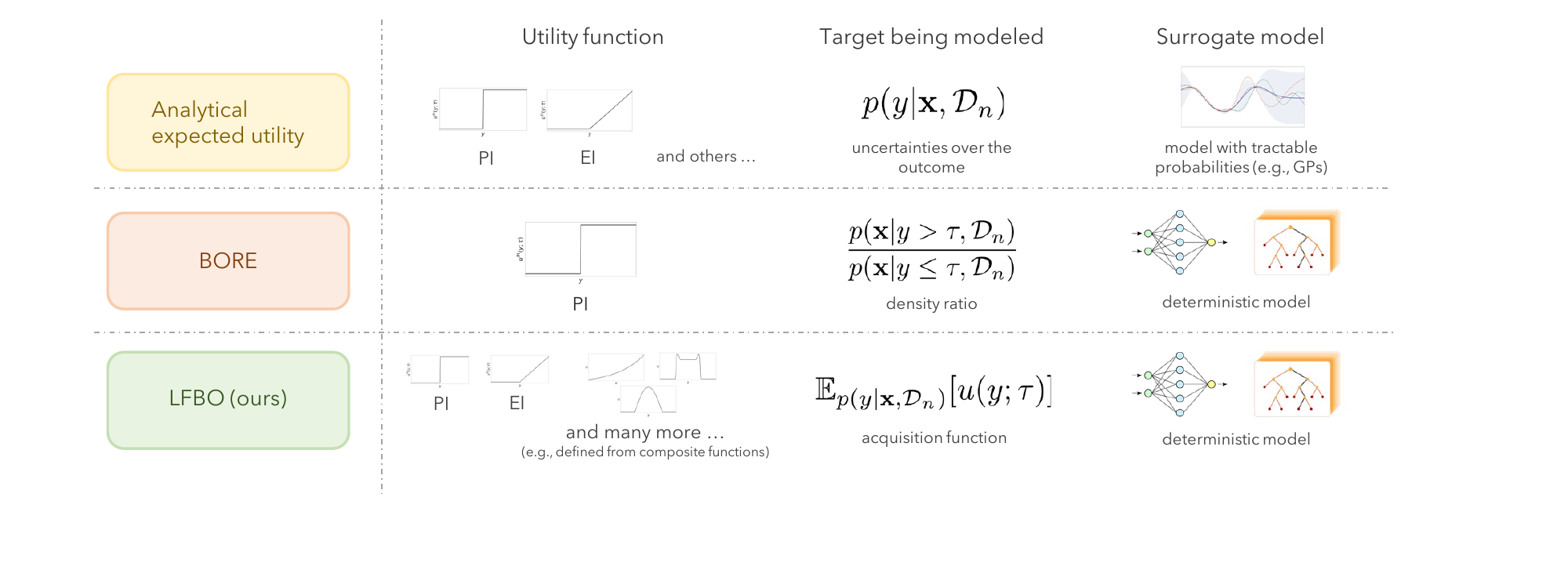}
    \caption{Overview of several types of acquisition functions for BO. (\textit{Top}) Closed-form acquisition functions in the form of expected utility often require certain utility functions (\textit{e.g.}, PI and EI) and certain surrogate models (\textit{e.g.}, GPs). (\textit{Middle}) BORE~\citep{tiao2021bore} uses simple, deterministic models (\textit{e.g.}, neural networks, decision trees) to model a density ratio with likelihood-free inference, but this limits its utility function to the one associated with PI. (\textit{Bottom}) Our likelihood-free BO approach extends BORE to any non-negative utility function, including the one associated with EI and others whose expectations may not have a closed-form solution (\textit{e.g.}, composite function EI in~\citep{astudillo2019bayesian}).
    }
    \label{fig:overview}
\end{figure*}

As discussed earlier, the acquisition function is constructed by two components, the surrogate model and the utility function\footnote{In the case of DR, the utility function is implicitly chosen.}, over which restrictions are often placed to ensure that the acquisition function is tractable and efficient to compute.
To allow tractable inference, EU acquisition functions often choose models such that $p(y | \rvx, \gD)$ is explicitly defined (\textit{e.g.}, Gaussian), which excludes other popular models such as implicit, simulator-based models~\citep{diggle1984monte} that do not have tractable likelihood estimates. While DR acquisition functions do not pose such limitations, they are limited in terms of the utility function and cannot effectively leverage knowledge of the structure %
of function $g(\rvx)$ in a grey-box BO setting~\citep{astudillo2022thinking}.

Inspired by how likelihood-free variational inference (VI) extends VI to implicit models, %
we develop ``likelihood-free'' acquisition functions that extend BO to more general models and utility functions. First, we introduce a formal definition for when a likelihood-free acquisition function admits a certain utility function (Section \ref{sec:consistency}).
We then develop likelihood-free acquisition functions that can admit any non-negative utility function (Section~\ref{sec:lfbo}) and discuss a special case that amounts to weighted classification (Section~\ref{sec:practical}). Finally, we discuss the representation capability of likelihood-free acquisition functions (Section~\ref{sec:capacity}).

\subsection{Equivalence between Acquisition Functions}
\label{sec:consistency}

Before we establish our likelihood-free BO method, we introduce an equivalence relationship among acquisition functions that describes if two acquisition functions would behave identically under infinite observations, even with different types of surrogate models.
This allows us to replace the expected utility acquisition function with an equivalent likelihood-free approach in practical scenarios, where the difference in exploration-exploitation trade-offs would only occur from the difference in models fitted from finite queries.

\begin{definition}[Equivalence between acquisition functions] 
Assume that the statistical models that we consider are consistent (satisfying \autoref{assu:consistency-full} in \autoref{app:proofs}).
For all $n \in \bb{N}$, let $\gD^{(1)}_n$ and $\gD^{(2)}_n$ denote the random variables representing $n$ non-\textit{i.i.d.} queries from two valid sequential decision making processes for the same black-box function (satisfying \autoref{def:valid-sdp} in \autoref{app:proofs}). Let $\mathrm{Ac}^{(1)}$ and $\mathrm{Ac}^{(2)}$ denote two acquisition functions with the same search space $\gX$ and hyperparameter space $\gT$. %
We say that ``$\mathrm{Ac}^{(1)}$ is equivalent to $\mathrm{Ac}^{(2)}$'' if $\forall \rvx \in \gX, \tau \in \gT$, a positive constant factor $\alpha$, and a constant $\beta$, the following is true: %
\begin{align*}
   \plim_{n \to \infty} \mathrm{Ac}^{(1)}(\rvx;  \gD^{(1)}_n, \tau) = \alpha \plim_{n \to \infty} \mathrm{Ac}^{(2)}(\rvx; \gD^{(2)}_n, \tau) + \beta. %
\end{align*}
\end{definition}
We note that for a sequential decision process to be valid (\autoref{def:valid-sdp}), it needs to place non-zero probabilities to every point in the search space. Since any $\epsilon$-greedy strategy with $\epsilon > 0$ would satisfy this, our condition here is reasonable (as we can apply $\epsilon$-greedy strategy to any decision process based on acquisition functions). Moreover, a similar consistency assumption is often made for GP-based methods~\citep{williams2006gaussian}. %

With this definition, the claim that DR acquisition functions being equivalent to the EI acquisition function \citep{bergstra2011algorithms,tiao2021bore} may be interpreted as ``$L^{\DR}$ is equivalent to $\gL^{\EI}$''. Unfortunately, this interpretation is inaccurate according to our definitions. In \autoref{prop:consistency-pi}, we introduce an equivalence result between DR and PI. This echos the claim made in \citet{garnett_bayesoptbook_2022}. %

\begin{restatable}{proposition}{consistencypi}\label{prop:consistency-pi}
$L^{\DR}$ is equivalent to $\gL^{\PI}$.
\end{restatable}

\begin{proof}
In \Cref{app:dr-proofs}.
\end{proof}
The above statement indicates that the density ratio acquisition functions are equivalent to probability of improvement under the same $\tau$, which justifies its practical effectiveness.
However, as PI and EI are clearly different acquisition functions, DR cannot be equivalent to EI\footnote{We note that in \citet{tiao2021bore}, the ``equivalence'' between DR and EI are established only for specific surrogate models.}. We include a formal statement in \autoref{coro:inconsistency-ei} and its proof in Appendix~\ref{app:proofs}.  

Another intuitive argument for the non-equivalence between DR and EI is through the weights of the observed $y$ in the acquisition function: for a fixed threshold $\tau$, DR would treat all $(\rvx, y)$ pairs above it with equal importance (just like how PI does), whereas EI also considers how much $y$ is greater than $\tau$ and places more importance when $(y - \tau)$ is higher.

\subsection{A General Recipe for Likelihood-free BO}
\label{sec:lfbo}
Given %
the above results, it is natural to ask if there exists another likelihood-free acquisition function that is actually equivalent to EI. Here, we provide a general approach to obtain likelihood-free acquisition functions that are equivalent to acquisition functions based on any non-negative expected utility function, which naturally includes both EI and PI.

Our approach is inspired by density ratio estimation via variational $f$-divergence estimation~\citep{nguyen2008estimating} but generalizes it beyond the ratio between densities. The following lemma provides a variational representation for the expected utility at any point $\rvx$, provided with samples from some $p(y | \rvx)$ (that may or may not depend on data). This enables us to replace integration (that are potentially intractable for complex $p(y | \rvx)$ and $u$) with a variational objective (that can be performed on samples). 
\begin{restatable}[Variational representation of an integral]{lemma}{variationalintegral}\label{lemma:variational-integral}
For all non-negative utility functions $u: \R \times \gT \to [0, \infty)$, and for all strictly convex functions $f: [0, \infty) \to \R$ with third order derivatives,  
\begin{align*}
\bb{E}_{p(y | \rvx)}[u(y; \tau)] & = \\
\argmax_{s \in [0, \infty)} \ & \bb{E}_{p(y | \rvx)}[u(y; \tau) f'(s)] - f^\star(f'(s)),
\end{align*}
where $f^\star$ is the convex conjugate of $f$, and the maximization is performed over the $s \in [0, \infty)$ (the domain of $f$). %
\end{restatable}
\begin{proof}
In \Cref{app:lfbo-proofs}.
\end{proof}

By minimizing a variational objective averaged under any distribution that is supported on the search space $\gX$, we can recover an acquisition function over $\rvx$ that is equivalent to any expected utility acquisition function in the form of \autoref{eq:myopic-aqf-def}. This new acquisition function is likelihood-free, in the sense that it does not model distributions with a tractable probability and only uses samples from the observations $\gD_n$.
\begin{restatable}{theorem}{mainthm}\label{thm:main}
For $u$ and $f$ as defined in \autoref{lemma:variational-integral},
$\forall \rvx \in \gX$ and $\forall \tau \in \gT$,  
let $L_{\mathrm{LFBO}}^{(u)}(\rvx; \gD_n, \tau) = \hat{S}_{\gD_n, \tau}(\rvx)$, where
\begin{align}
     \hat{S}_{\gD_n, \tau} = \argmax_{S: \gX \to \R} \bb{E}_{\gD_n}[u(y; \tau) f'(S(\rvx)) - f^\star(f'(S(\rvx))]. \nonumber
\end{align} 
Then $L^{(u)}_{\mathrm{LFBO}}$ is equivalent to $\gL^{(u)}$ in \autoref{eq:myopic-aqf-def}. %
\end{restatable}
\begin{proof}
In \Cref{app:lfbo-proofs}.
\end{proof}
The above statement converts the process of obtaining the acquisition function into an optimization problem that only relies on samples, such that the optimized model $\hat{S}_{\gD_n, \tau}$ can directly be used as the acquisition function.
Therefore, we have obtained a general approach to construct acquisition functions that are (\textit{i}) likelihood-free and (\textit{ii}) equivalent to an expected utility acquisition function. Due to the likelihood-free nature of the method, we name our approach Likelihood-free Bayesian Optimization (\methodshort{}).

\subsection{Practical Likelihood-free Acquisition Functions}\label{sec:practical}
From \autoref{thm:main}, we can further choose specific convex $f$ and recover practical objective functions such as a weighted classification one.  
\begin{corollary}[Weighted classification] \label{coro:logistic}
For $u$ as defined in \autoref{lemma:variational-integral}, $\forall \rvx \in \gX$ and $\forall \tau \in \gT$, 
let
\begin{align}
    L_{\mathrm{C}}^{(u)}(\rvx; \gD_n, \tau) = \hat{C}_{\gD_n, \tau}(\rvx) / (1 - \hat{C}_{\gD_n, \tau}(\rvx)),
    \label{eq:classification-acq}
\end{align}
where $\hat{C}_{\gD_n, \tau}$ is the maximizer to the following objective over $C: \gX \to (0, 1)$:
\begin{align}
     \bb{E}_{(\rvx, y) \sim \gD_n}[u(y; \tau) \log C(\rvx) + \log (1 - C(\rvx))], \label{eq:binary-classification}
\end{align}
then $L_{\mathrm{C}}^{(u)}$ is equivalent to $\gL^{(u)}$ in \autoref{eq:myopic-aqf-def}. %
\end{corollary}
\begin{proof}
Let $f(r) = r \log \frac{r}{r + 1} + \log \frac{1}{r + 1}$ for all $r > 0$, then $f^\star(c) = - \log (1 - \exp(c))$, $f'(r) = \log r - \log (r + 1)$, and $f''(r) = 1/r - 1/(r+1) > 0$ (so $f$ is convex, and its third order derivative exists). Applying $f(r)$ and $S = C / (1 - C)$ to \autoref{thm:main} completes the proof.
\end{proof}
From \autoref{coro:logistic}, we can optimize a classification objective (\autoref{eq:binary-classification}) with weights defined by the utility function and expect it to behave similarly to the EU acquisition functions in \autoref{eq:myopic-aqf-def}.

\paragraph{\methodshort{} for PI} For PI, $u^{\PI}(y; \tau) := \bb{I}(y - \tau > 0)$, \autoref{eq:binary-classification} becomes:
    $$
\frac{|\gD_n^+|}{|\gD_n|}\underbrace{\bb{E}_{(\rvx, y) \sim \gD_n^{+}}[\log C(\rvx)]}_{\mathrm{positive}} + \underbrace{\bb{E}_{(\rvx, y) \sim \gD_n}[\log (1 - C(\rvx))]}_{\mathrm{negative}},
    $$
    where $\gD_n^+$ denotes the subset of $\gD_n$ where $y > \tau$. Similar to word2vec~\citep{mikolov2013distributed}, we can treat samples in the first term as positive and samples in the second term as negative. Our objective encourages $C$ to assign higher values to $\rvx$ in positive samples and lower values to $\rvx$ in negative samples; all the positive samples have the same weight, which is also the case in BORE.
    
\paragraph{\methodshort{} for EI}    For EI, $u^{\EI}(y; \tau) := \max(y - \tau, 0)$, \autoref{eq:binary-classification} becomes:
    $$
\frac{|\gD_n^+|}{|\gD_n|}\underbrace{\bb{E}_{\gD_n^{+}}[(y - \tau) \log C(\rvx)]}_{\mathrm{positive}} + \underbrace{\bb{E}_{\gD_n}[\log (1 - C(\rvx))]}_{\mathrm{negative}},
    $$
    where positive samples are weighted by $(y - \tau)$ times a constant ($|\gD_n^+|/|\gD_n|$). Different from PI, here we favor $y$ that are not only above the threshold $\tau$, but also have higher values. In practice, the factor $|\gD_n^+|/|\gD_n|$ can be removed, as this amounts to simply scaling the utility function by a constant (since the utilities/weights for the samples with $y < \tau$ are zero). We can also normalize the weights within the positive term such that the average %
    weight is one; we adopt this approach in the experiments.

\subsection{Representation Capability of Likelihood-free Acquisition Functions}\label{sec:capacity}
Unlike GPs that produce a distributional estimate given $\rvx$, models used in likelihood-free acquisition functions (\textit{e.g.}, $C(\rvx)$ in \autoref{eq:binary-classification}) would only produce a point estimate given $\rvx$. Despite the lack of explicit uncertainty estimates over $y$, likelihood-free acquisition functions are not limited in what they can represent; for example, wide-enough neural networks can approximate any acquisition function that is Borel measurable~\citep{hornik1989multilayer}. 
Conversely, even for the same utility $u$, there may exist many different surrogate models of $p(y | \rvx, \gD)$ that give the same acquisition value (see an argument for Gaussian distributions in \autoref{prop:gauss-acqf}). Thus, our likelihood-free method is a more direct approach to acquiring the acquisition function without loss of representation capabilities.

\section{Related Work}

\paragraph{Bayesian Optimization} Many works in Bayesian Optimization adopt GPs as the underlying surrogate model and some form of expected utility as the acquisition function. The utilities in some acquisition functions %
can be expectations themselves (such as ES). Some works aim to make the posterior inference process more scalable, possibly sacrificing the tractability or the exactness of the posterior; these include sparse online GPs~\citep{mcintire2016sparse}, linear models~\citep{perrone2018scalable}, neural network ensembles~\citep{white2019bananas}, Bayesian neural networks~\citep{springenberg2016bayesian}, Bayesian linear models and GPs with neural feature learners~\citep{snoek2015scalable, tran2020methods}, and GPs with deep kernels from meta learning~\cite{wistuba2021few}. Nevertheless, these methods replace GPs with another model and do not change how they interact with the acquisition functions. TPE~\citep{bergstra2011algorithms} is an exception in that it first estimates probability distributions over $\rvx$ (as opposed to over $y$ given $\rvx$) and then obtains the acquisition function through density ratios over $\rvx$. Ratio estimation approaches have also been combined with mutual information acquisition functions for Bayesian experimental design~\citep{kleinegesse2019efficient,zhang2021scalable}. %

The work most related to ours is BORE~\citep{tiao2021bore}, which applies likelihood-free density ratio estimation to the TPE acquisition function. Our approach extends BORE over two aspects: (\textit{i}) we formally identify that BORE (and equivalently, TPE) is not equivalent to EI, and that it is only equivalent to the PI utility function; (\textit{ii}) we derive a likelihood-free approach that applies to not only EI, but also general expected utility functions.

\paragraph{Likelihood-free inference} Likelihood-free inference methods have been widely applied to implicit models where tractable forms of the likelihood do not exist, such as generative adversarial networks (GANs,~\citet{goodfellow2014generative}) and hierarchical implicit models~\citep{tran2017hierarchical}. These are often posed as density/likelihood ratio estimation problems from samples. Various methods have been proposed beyond direct kernel density estimation approaches, such as KL importance estimation procedure (KLIEP, \citet{sugiyama2007direct,sugiyama2012density}), kernel mean matching (KMM, \citet{gretton2009covariate}), and unconstrained least-squared importance fitting (ULSIF, \citep{kanamori2010theoretical}).
The likelihood-ratio estimation approach used in GANs can be interpreted as class probability estimation based on proper scoring rules~\citep{buja2005loss,gneiting2007strictly}, variational $f$-divergence estimation~\citep{nguyen2008estimating}, or ratio matching~\citep{sugiyama2013density,uehara2016generative}. These perspectives are known to be equivalent to one another~\citep{reid2011information,mohamed2016learning} and are also applied in various domains such as generative modeling~\citep{gutmann2010noise,song2020bridging,yu2020training}, domain adaptation~\citep{bickel2007discriminative}, inverse reinforcement learning~\citep{yu2019multi}, and black-box sequence design~\citep{zhang2021unifying}. 

Our likelihood-free acquisition function is different from these prior approaches (such as class probability estimation used in BORE) in that we are not necessarily estimating the ratio between two probability measures defined over $\gX$ (except for the case of PI), but the ratio between a general non-negative measure (defined as some expected utility over the positive samples) and a probability measure (over the negative samples). Nevertheless, our approach is closely related to these perspectives, and can be treated as a classification problem with unnormalized weights determined by the utility values.

\section{Experiments}

In this section, we evaluate our likelihood-free BO approach on multiple synthetic and real-world optimization tasks. We first verify our theory by comparing against ground truth targets in a synthetic case (Section~\ref{sec:exp:theory}). Next, we compare against various baselines, especially BORE~\citep{tiao2021bore},  illustrating the advantages of the EI utility function (as well as others) in conjunction with our likelihood-free BO approach (Section~\ref{sec:exp:pi-vs-ei}). Finally, we illustrate how we can massively improve performance by integrating \methodshort{} with existing knowledge about the objective function, specifically with a composite, grey-box function optimization problem (Section~\ref{sec:exp:comp}). 
Here, \methodshort{} uses the weighted classification approach in~\autoref{eq:binary-classification},~\autoref{coro:logistic}; we leave the possibilities of applying different $f$ in~\autoref{thm:main} as future work. For all BO experiments, we first query $10$ random candidates uniformly over the search space, and then proceed with the BO methods. %

\subsection{Verifying the Equivalence Relationship of Likelihood-free Acquisition Functions}
\label{sec:exp:theory}

\begin{figure}
    \centering
    \includegraphics[width=0.47\textwidth]{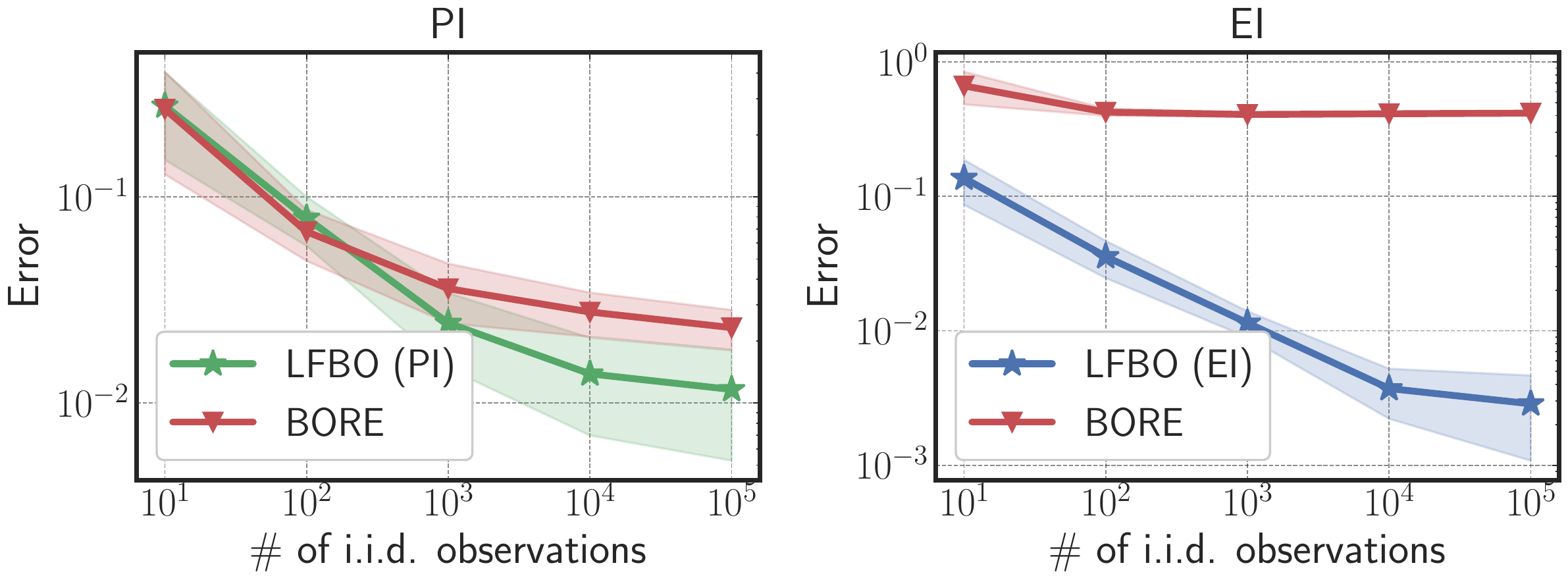}
    \caption{$L_1$ distance between the modeled acquisition functions and the ground truth. $x$ is drawn uniformly between $[-1, 1]$. The shaded regions show the mean plus and minus one standard deviation across different seeds. The ground truth acquisition functions for the left and right figures are Probability of Improvement and Expected Improvement respectively.}
    \label{fig:synthetic-equivalence-main}
\end{figure}

We first validate the theory over our proposed likelihood-free acquisition functions over a synthetic problem, where the search space is $x \in [-1, 1]$ and observations are obtained from $g(x) = - \sin(3x) - x^2 + 0.6x$ plus independent Gaussian noise $\epsilon \sim \gN(0, 0.01)$ (see \autoref{fig:demo} for the function). This allows to evaluate the ``ground truth'' EI and PI values for any $x$ in the search space (similar to how EI and PI are computed for GPs). We estimate the ground truth EI and PI with our \methodshort{} acquisition functions, with BORE~\citep{tiao2021bore} as a baseline. The classifier model is a two layer fully-connected neural network with $128$ units at each layer; we repeat all experiments with $5$ random seeds. As our goal here is to validate the consistency claims, we only take \textit{i.i.d.} samples uniformly and do not perform any adaptive queries.

If a likelihood-free acquisition function is indeed equivalent to PI or EI, then we should expect it to become closer to the ground truth as the number of observations $n$ increases. We report the quality of the acquisition functions, measured in $L_1$ error from the ground truth, in \autoref{fig:synthetic-equivalence-main} (more in~\autoref{fig:synthetic-equivalence-appendix}, Appendix~\ref{app:sec:syn}). For \methodshort{}, both errors are roughly inversely proportional to $n$, whereas the same can only be said for BORE and PI; the error between BORE and EI remains large even with many samples. Therefore, the empirical results align with our statements in \autoref{prop:consistency-pi}, \autoref{thm:main}, and \autoref{coro:logistic}.

\subsection{Real-world Benchmarks}
\label{sec:exp:pi-vs-ei}

\begin{figure*}
    \centering
    \includegraphics[width=.95\textwidth]{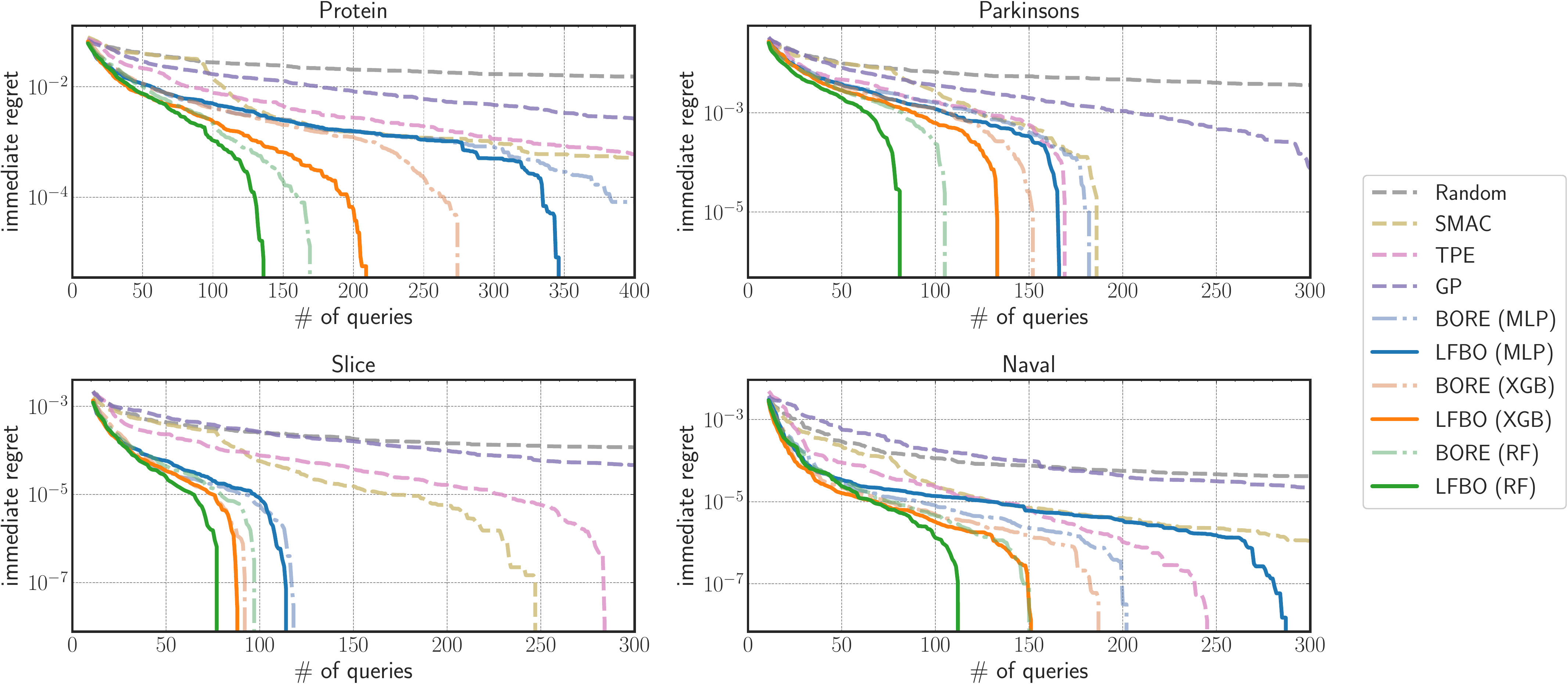}
    \vspace{-10pt}
    \caption{Immediate regret on the HPOBench neural network tuning problems. LFBO uses the utility function for EI. Each curve is the mean of results from 100 random seeds.}
    \label{fig:hpobench}
    \vspace{-5pt}
\end{figure*}

\begin{figure*}
    \centering
    \includegraphics[width=.95\textwidth]{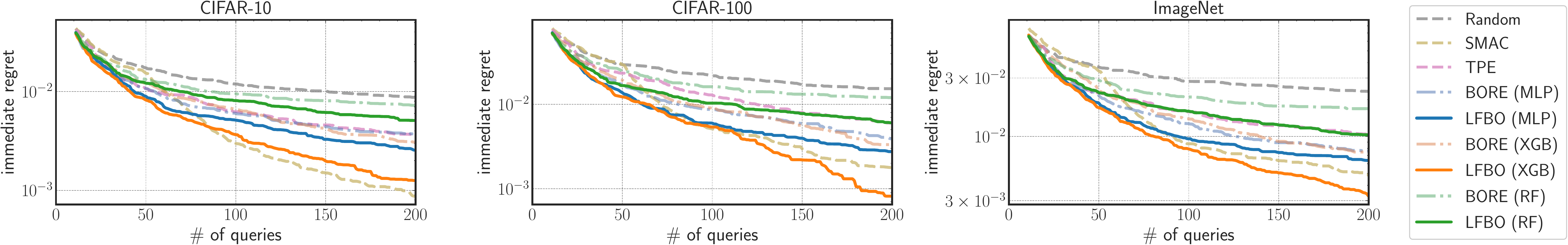}
    \vspace{-10pt}
    \caption{Immediate regret on the NAS-Bench-201 neural architecture search problems. LFBO uses the utility function for EI. Each curve is the mean of results from 100 random seeds.}
    \label{fig:nasbench}
    \vspace{-5pt}
\end{figure*}

\textbf{Setup.} We empirically evaluate our method on various real-world optimization tasks, including hyperparameter optimization of neural networks (HPOBench, \citet{klein2019tabular}) and neural architecture search (NAS-Bench-201, \citet{dong2020bench}). We provide more details for these experiments in \Cref{app:exp-real}.

In HPOBench, we aim to find the optimal hyperparameter configurations for training a two-layer feed-forward neural networks on four popular UCI datasets \citep{asuncion2007uci} for regression: protein structure (Protein), slice
localization (Slice), naval propulsion (Naval) and parkinsons telemonitoring (Parkinsons). The configuration space includes batch size, initial learning rate, learning rate schedule, as well as layer-specific dropout rates, number of units and activation functions. For each dataset, HPOBench tabulates the mean squared errors of all $62,208$ configurations.

In NAS-Bench-201, the network architecture is constructed as a stack of neural cells, each represented as a directed acyclic graph (DAG), and we focus on the design of the neural cell, \textit{i.e.}, assigning an operation to 6 edges of the DAG from a set of 5 operations. Similar to HPOBench, this dataset also provides a lookup table that tabulates the training results of all $15,625$ possible configurations for three datasets: CIFAR-10, CIFAR-100 \citep{krizhevsky2009learning} and ImageNet-16 \citep{chrabaszcz2017downsampled}. 

In LFBO, we need to train a probabilistic classifier %
for candidate suggestion at each iteration. Here, we consider Multi-layer perceptions (MLP), Random Forest (RF, \citet{breiman2001random}), and gradient boosted trees (XGBoost, \citet{chen2016xgboost}).
The MLP-based models are differentiable with respect to its inputs %
but requires special care for categorical variables (\textit{e.g.}, with one-hot vectors). In contrast, the tree-based ensemble classifiers can naturally deal with the discrete inputs but are non-differentiable. 
Despite this, for optimizing the acquisition functions, we use random search for all methods on all of the datasets, \textit{i.e.}, first randomly sampling a large batch of candidates and then selecting the one with the maximum acquisition value in \autoref{eq:classification-acq}.

We compare our approach with a comprehensive set of state-of-the-art black-box optimization methods, including BORE \citep{tiao2021bore}, Tree-structured Parzen Estimator (TPE, \citet{bergstra2011algorithms}), Sequential Model-based Algorithm Configuration (SMAC, \citet{hutter2011sequential}), GP-based Expected Improvement \citep{jones1998efficient,balandat2020botorch} and Random Search \citep{bergstra2012random}. Since LFBO with the utility function for PI would behave similarly to BORE, we consider LFBO with the utility function for EI in these experiments. Our implementation is based on an implementation of BORE, and we inherit all its default hyperparameters. Notably, we select the threshold $\tau$ to be the $(1 - \gamma)$-th quantile of the observed values, which includes less positive samples as $\gamma$ becomes smaller. We select $\gamma = 0.33$\footnote{We perform a brief ablation study on this parameter in Appendix~\ref{app:sec:ablation-gamma}, showing that LFBO (EI) is less sensitive to it.}, following \citep{tiao2021bore}.

\textbf{Results.} Following \citep{tiao2021bore}, we report the immediate regret (\textit{i.e}. the absolute error between the global minimum of each benchmark and the lowest function value observed so far)\footnote{We note that here and in \citep{tiao2021bore}, the actual output from the black-box function is a random value from several possible ones (simulating noise in the optimization process), whereas the global minimum considers the expectation over the values. Therefore, a negative immediate regret is possible (and we would use the corresponding neural network for deployment). The flat vertical lines in our plots mean that the average immediate regret (across all the seeds) has become negative.} for quantitative evaluation of each method. The results for HPOBench and NAS-Bench-201 are summarized in \autoref{fig:hpobench} and \autoref{fig:nasbench} respectively, where the results are averaged across \textit{100 random seeds}. We can see that our LFBO approach consistently achieves the best result across all the datasets, except on NAS-Bench-201-CIFAR-10 where the SMAC method is slightly better. 

Notably, with all three classifier implementations (MLP, RF and XGBoost), LFBO (EI) consistently outperforms its BORE counterparts except for Naval with MLP, which demonstrates the effectiveness of choosing EI as the acquisition function and the value of enabling EI in the likelihood-free scenario. %
Moreover, we observe that GP-based EI methods (where we use Matern-Kernel with default hyperparameters values in the BoTorch\footnote{\url{https://github.com/pytorch/botorch/blob/main/botorch/models/gp_regression.py}} library) %
perform poorly on both benchmarks (omitted on NAS-Bench-201 where its performance is close to random search), possibly due to the complexity of attempting to learn a full distributional description of the underlying function. These results show the great potential of our generalized likelihood-free BO paradigm for strategic sequential decision making tasks.

\begin{figure}
    \centering
    \includegraphics[width=0.47\textwidth]{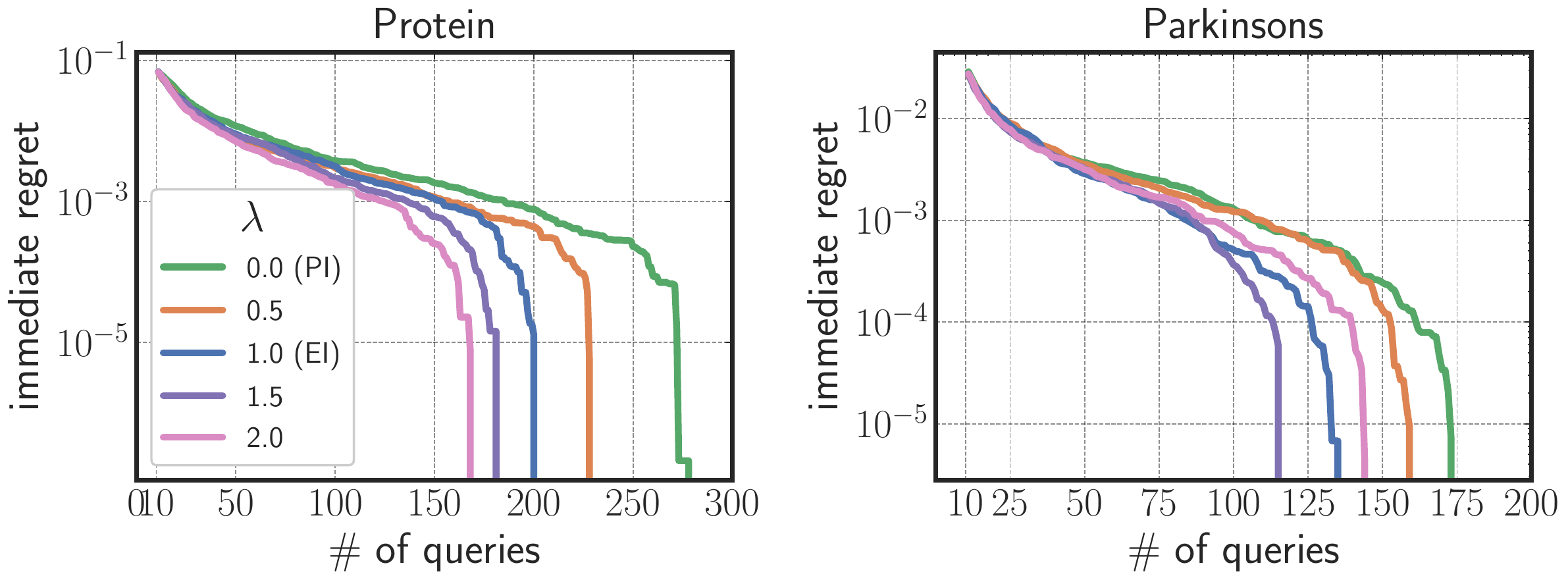}
    \caption{Results on two HPOBench problems with different utility functions, controlled by $\lambda$. %
    }
    \label{fig:ablation_lambda}
\end{figure}

\textbf{Alternative utility functions.} Since \methodshort{} allows any non-negative utility function, we may consider other alternatives to EI and PI. For example, we consider the following function:
$u^{\lambda}(y; \tau) = (y - \tau)^\lambda$ if $y > \tau$; $0$ otherwise. $u^{\lambda}$ generalizes EI when $\lambda = 1$ and PI when $\lambda = 0$. \autoref{fig:ablation_lambda} illustrate the results on two HPOBench problems with $\lambda \in \{0.0, 0.5, 1.0, 1.5, 2.0\}$ and the XGBoost classifier. These results show that performance on these tasks increases as $\lambda$ goes from $0$ (PI) to $1$ (EI), and that further improvements can be achieved by considering $\lambda$ greater than $1$; in this case, the utility with $\lambda = 1.5$ outperforms the EI one on both problems.

\subsection{LFBO with Composite Objective Functions}
\label{sec:exp:comp}

Here, we consider the case of composite objective functions~\citep{astudillo2019bayesian,maddox2021bayesian}, where $g(\rvx)$ is a known transformation of a vector-valued black-box function $h(\rvx): \gX \to \R^d$, \textit{e.g.}, $$g(\rvx) = -\norm{h(\rvx) - z_\star}_2^2,$$ and $z_\star \in \R^d$ is a vector observed from the real world. Compared to standard BO approaches that do not exploit this information, a grey-box BO method that explicitly models $h$ can make a much better sampling decision~\citep{astudillo2022thinking}. However, if we model $h(\rvx)$ with GPs, then the EI acquisition function defined over the above $g$ is no longer tractable, and traditional approaches would rely on closed-form but biased approximations~\citep{uhrenholt2019efficient} or unbiased but slow Monte Carlo estimates~\citep{astudillo2019bayesian}.

In \methodshort{}, we can also exploit the composite nature of the objective function to obtain more efficient BO algorithms. Specifically, we may consider a composite neural network $C_\theta(\rvx): \gX \to \R$ parametrized by training parameters $\theta$: $$C_\theta(\rvx) = \frac{u(-\norm{h_\theta(\rvx) - z_\star}_2^2; \tau)}{u(-\norm{h_\theta(\rvx) - z_\star}_2^2; \tau) + 1},$$ where we use a neural network\footnote{While other models (such as trees) can be used here in principle as well, we focus on neural networks as gradient-based optimization of the likelihood-free objective function (training) is more straightforward.} for $h_\theta(\rvx): \gX \to \R^d$. As long as $u$ allows automatic differentiation (\textit{e.g.}, in the case of EI), we can apply backpropagation over $\theta$ to optimize it.

We consider a well-known test problem in the Bayesian calibration literature~\citep{bliznyuk2008bayesian} that models the concentration of substances in a chemical accident. Our goal is to find the underlying parameters, such as the mass of pollutants, that minimizes the sum of squared errors with the real-world observation at 12 locations. We include more details in Appendix~\ref{app:sec:composite}. 

Results of this experiment %
are in \autoref{fig:comp-regret}. Our composite \methodshort{} (EI) approach performs significantly better than regular GP and LFBO, while being on par with a GP approach based on composite EI~\citep{astudillo2019bayesian}. Moreover, once our neural network has been trained, we do not require sampling from the surrogate to evaluate the composite acquisition function (unlike the GP counterpart which relies on Monte Carlo samples). Even with training taken into account, our composite \methodshort{} method is much more efficient (around $15\times$ to $30\times$) when evaluated on the same hardware. 

\begin{figure}
    \centering
    \includegraphics[width=0.42\textwidth]{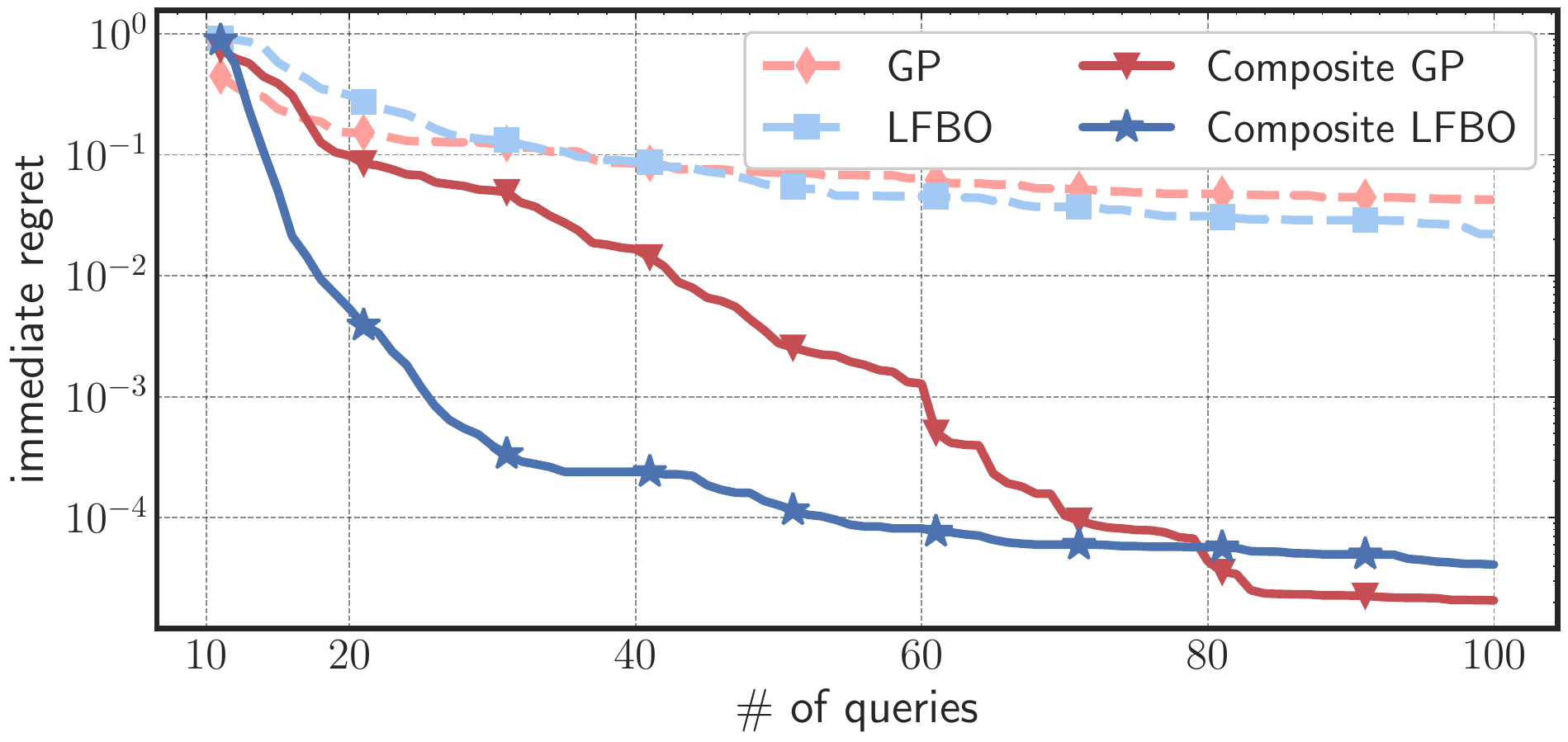}
    \caption{Results on composite objective function optimization. The utility function for EI is used by all the methods considered.}
    \label{fig:comp-regret}
\end{figure}

\section{Discussions}

We have introduced likelihood-free BO (\methodshort{}) that extends BO to general surrogate models and utility functions. Motivated by likelihood-free inference, \methodshort{} directly estimates the acquisition function from samples. %
With infinite query efforts in the search space, \methodshort{} will converge to the desired expected utility function.
\methodshort{} can be reduced to learning the optimal classifier of a weighted classification problem. %
Empirical results demonstrate both the validity of our theory and the usefulness of \methodshort{} (EI) empirically in black-box and composite function optimization settings. %

\paragraph{Limitations} While we observe some empirical successes with LFBO, we believe that it has certain limitations.
\begin{itemize}
    \item First, because LFBO does not explicitly learn a distributional representation about our current knowledge of the black-box function, the behaviors of the algorithm (explore-exploit trade-off) may be less interpretable. 
    \item Second, LFBO combines two separate steps (learning the surrogate model, and maximizing the acquisition function based on the learned surrogate model) into one classification problem, thus the trained model is tied to one specific acquisition function, while GP-based methods can be more flexible since the surrogate model and the acquisition function are decoupled.
    \item Third, LFBO has to choose a threshold below the current optimal observation, since otherwise no observations will fall into the positive class, and the model will not learn anything useful. This is also tied to the fact that LFBO models do not provide a probabilistic interpretation to the underlying function.
    \item Finally, there is a risk where the acquisition function from LFBO is over-confident, and would reduce exploration diversity (see the example on the Forrester function in \autoref{fig:forrester}).
\end{itemize}
We believe that these could be mitigated with approaches that still provide a probabilistic interpretation to the black-box function, but are not necessarily have tractable acquisition functions like GPs; we can use LFBO as trying use a model to amortize the acquisition function (that is otherwise intractable), which could be faster to optimize than using Monte Carlo samples. 

\paragraph{Future work} Apart from applying \methodshort{} to various application domains, there are several interesting avenues for future works:
\begin{itemize}
    \item Extensions to more sophisticated BO paradigms, such as batch-based, multi-fidelity~\citep{forrester2007multi}, multi-objective~\citep{hernandez2016predictive}, and complex composite functions~\citep{astudillo2021bayesian}.
    \item Integration of \methodshort{} with explicit and implicit surrogate functions as an amortization approach, which may accelerate approaches such as knowledge gradients~\citep{frazier2008knowledge}.
    \item \methodshort{} with alternative choices of the scoring rule (defined via $f$), better heuristics of selecting the threshold $\tau$, and additional techniques that bridge the gap between \methodshort{} and GP-based BO.
    \item Exploring LFBO with other information acquired when the exact function value is unknown, such as ordering relationships~\citep{christiano2017deep}.
\end{itemize}

\section*{Acknowledgements}
The authors would like to thank Louis Tiao, Roman Garnett, Eytan Bakshy, and the anonymous reviewers for helpful feedback, as well as Chris Yeh and Raul Astudillo for spotting a typo.
This research was supported by NSF (\#1651565, \#1522054, \#1733686), ONR (N00014-19-1-2145), AFOSR (FA9550-19-1-0024), and Amazon AWS.

\bibliography{references}
\bibliographystyle{icml2022}

\newpage
\appendix
\onecolumn

\section{Proofs}
\label{app:proofs}

\subsection{Definitions and Assumptions}

\begin{definition}[A valid sequential decision making process]\label{def:valid-sdp}
Let $p_t(\rvx | \gD_{t-1})$ denote the distribution over $\rvx$ at time $t$ given observations $\gD_{t-1} = \{(\rvx_i, y_i)\}_{i=1}^{t-1}$ (which depends on the decision making process), and let $p(y | \rvx)$ be the conditional distribution (which depends on the true black-box function $g$); then we define a sequential and Markovian decision making process as follows:
\begin{gather}
    \rvx_t \sim p_t(\rvx | \gD_{t-1}), y_t \sim p(y | \rvx), \gD_{t} = \gD_{t-1} \cup \{(\rvx_t, y_t)\}, \gD_{0} = \varnothing. \label{eq:sdp}
\end{gather}
Define $p_t(\rvx)$ as the marginal distribution of the selected $\rvx$ at time $t$; a sequential decision making process is valid as long as $\plim_{t \to \infty} p_t(\rvx)$ exists and is supported on the search space $\gX$. We would commonly denote this limiting distribution as $p(\rvx)$.
\end{definition}

\begin{remark}
An $\epsilon$-greedy strategy generates a valid sequential decision making process (which would contain a mixture component that is the uniform distribution over $\gX$). Moreover, in both expected utility and density ratios, the marginal distribution over $\rvx$ can be canceled out, so it does not affect the acquisition functions discussed in this paper.  %
\end{remark}

The existence of a limiting ``ground truth'' distribution supported over the entire search space allows us to define to whose statistics our model estimators are converging to in probability.
\begin{assumption}[Consistency of statistical models]\label{assu:consistency-full}
Let $\gD_n = \{(\rvx_i, y_i)\}_{i=1}^{n}$ be a set of $n$ observations drawn from a valid sequential decision making process, whose limiting distribution is defined as $p(\rvx, y)$. For any statistical model that we consider, let $E_n(\gD_n)$ be the estimator of a statistic $E$ of $p(\rvx, y)$, then:
\begin{align}
    \plim_{n \to \infty} E_n(\gD_n) = E.
\end{align}
For example:
\begin{gather}
    \plim_{n \to \infty} p(y | \rvx, \gD_n) = p(y | \rvx), \quad \forall \rvx \in \gX, y \in \R \\
    \plim_{n \to \infty} p(\rvx| \gD_n) = p(\rvx), \quad \forall \rvx \in \gX \\
    \plim_{n \to \infty} p(\rvx | y > \tau, \gD_n) = p(\rvx | y > \tau), \quad \forall \rvx \in \gX, y \in \R, \tau \in \R.
\end{gather}
\end{assumption}

\subsection{Statements Regarding Density Ratios}
\label{app:dr-proofs}

These results are also discussed by this textbook~\citep{garnett_bayesoptbook_2022}, albeit without an explicit proof. %

\consistencypi*
\begin{proof}
From the consistency assumptions over $p(y | \rvx, \gD_n)$, we have that $\forall \rvx \in \gX, y \in \R$:
\begin{align}
    \plim_{n \to \infty} \gL^{\PI}(\rvx, \gD_n, \tau) &= \plim_{n \to \infty} \bb{E}_{p(y | \rvx, \gD_n)}[\mathbb{I}(y > \tau)] \\
    &= \plim_{n \to \infty} \int_{y=\tau}^{\infty} p(y | \rvx, \gD_n) \diff y \\
    &= \plim_{n \to \infty}  p(y > \tau | \rvx, \gD_n) = p(y > \tau | \rvx).
\end{align}
From the consistency assumptions over $p(\rvx | y > \tau, \gD_n)$ and $p(\rvx | \gD_n)$, we have that $\forall \rvx \in \gX, y \in \R$:
\begin{align}
   \gamma \plim_{n \to \infty} L^{\DR}(x; \gD_n, \tau) &= \gamma \plim_{n \to \infty}  \frac{p(\rvx | y > \tau, \gD)}{p(\rvx | \gD)} \\
    &= \gamma \frac{p(\rvx | y > \tau)}{p(\rvx)} \label{eq:slutsky-application} \\
    &= \frac{p(\rvx | y > \tau) p(y > \tau)}{p(\rvx)} = p(y > \tau | \rvx),
\end{align}
where $\gamma := p(y > \tau)$ does not depend on $\rvx$, and we apply Slutsky's theorem in \autoref{eq:slutsky-application}. Therefore,
\begin{align} 
\plim_{n \to \infty} L^{\DR}(\rvx; \gD_n, \tau) \propto \plim_{n \to \infty} \gL^{\PI}(\rvx; \gD_n, \tau) %
\end{align}
which completes the proof.
\end{proof}

\begin{restatable}{corollary}{inconsistencyei}\label{coro:inconsistency-ei}
$L^{\DR}$ is not equivalent to $\gL^{\EI}$.
\end{restatable}
\begin{proof}

Let $p(x) = \gN(0, 1)$, $p(y | x) = \gN(x, 1)$ defines a limit to the valid sequential decision process in \autoref{def:valid-sdp}, and let $\tau = 0$. Then applying the consistency assumptions over $p(y | x, \gD_n)$:
\begin{align}
    \plim_{n \to \infty} \gL^{\PI}(x; \gD_n, \tau) = \bb{E}_{p(y | x)}[\mathbb{I}(y > 0)] = \Phi(x),
\end{align}
where $\Phi(x)$ is the CDF of the standard Gaussian distribution. Meanwhile,
\begin{align}
    \plim_{n \to \infty} \gL^{\EI}(x; \gD_n, \tau) = \bb{E}_{p(y | x)}[\max(y, 0)] = x \Phi(x) + \phi(x),
\end{align}
where $\phi(x)$ 
is the PDF of the standard Gaussian distribution.
Applying \autoref{prop:consistency-pi}, we have that:
\begin{align}
    \plim_{n \to \infty} L^{\DR}(x; \gD_n, \tau) = \plim_{n \to \infty} \gL^{\PI}(x; \gD_n, \tau) \neq \plim_{n \to \infty} \gL^{\EI}(x; \gD_n, \tau),
\end{align}
for some $x$, which completes the proof.

We note that the statement directly contradicts previous claims regarding the relationship between EI and DR~\citep{bergstra2011algorithms,tiao2021bore}.
The following claim restates the results in \citet{bergstra2011algorithms,tiao2021bore}, which claim an equivalence between density (or likelihood) ratio acquisition functions and expected improvements\footnote{We note that \citet{bergstra2011algorithms,tiao2021bore} minimizes the BO objective in their formulation, whereas we paraphrase the claim to adapt to maximization.}:
\begin{claim}[\citep{bergstra2011algorithms,tiao2021bore}] Following notations in~\citet{tiao2021bore}, let $\tau$ be the $(1 - \gamma)$-th percentile of observed $y$ values (\textit{i.e.}, $\gamma = p(y > \tau)$), and let two densities $\ell(\rvx) := p(\rvx | y \leq \tau, \gD_n)$ and $g(\rvx) := p(\rvx | y > \tau, \gD_n)$ be the densities for $\rvx$ conditioned on $y$ being less or greater than the threshold, respectively. Then:
\begin{align}
\mathrm{EI}(\rvx, \tau) := \bb{E}_{p(y | \rvx, D_n)}[\max(y - \tau, 0)] \propto \frac{p(\rvx | y > \tau, \gD_n)}{(1 - \gamma) p(\rvx | y \leq \tau, \gD_n) + \gamma p(\rvx | y > \tau, \gD_n)}.
\end{align}
\end{claim}

\textit{Counter claim}. We will now go through the ``proof'' of the claim and identify where the argument does not hold. 

First we have that:
\begin{align}
    \mathrm{EI}(\rvx, \tau) = \bb{E}_{p(y | \rvx, D_n)}[\max(y - \tau, 0)] %
    & = \int_{\tau}^{\infty} \max(y - \tau, 0) p(y | \rvx, D_n) \diff y \\
    & = \frac{\int_{\tau}^{\infty} \max(y - \tau, 0) p(y | D_n) p(\rvx | y, D_n) \diff y}{p(\rvx | D_n)}, \label{eq:frac}
\end{align}
where we applied Bayes' rule in \autoref{eq:frac}. 
Then, the denominator evaluates as:
\begin{align}
    p(\rvx | D_n) &= p(\rvx | y  \leq \tau, D_n) p(y \leq \tau | D_n) + p(\rvx | y  > \tau, D_n) p(y > \tau | D_n) \\
    &= \ell(\rvx) p(y \leq \tau | D_n) + g(\rvx) p(y > \tau | D_n) = (1 - \gamma) \ell(\rvx) + (1 - \gamma) g(\rvx). \label{eq:denom}
\end{align}
where we use the definition that $\gamma := p(y > \tau | D_n)$. Finally, \citet{bergstra2011algorithms,tiao2021bore} claims that the numerator in \autoref{eq:frac} evaluates to:
\begin{align}
    \int_{\tau}^{\infty} \max(y - \tau, 0) p(y | D_n) p(\rvx | y, D_n) \diff y =  g(\rvx) \int_{\tau}^{\infty} \max(y - \tau, 0) p(y | D_n) \diff y \label{eq:numer}
\end{align}
which is $g(\rvx)$ times a value independent of $\rvx$. Dividing \autoref{eq:numer} with \autoref{eq:denom} and using \autoref{eq:frac} should recover the result. However, in \autoref{eq:numer}, $g(\rvx) := p(\rvx | y > \tau)$ is directly taken out of the integral, as \citet{bergstra2011algorithms} assumed it is independent of $y$ as long as $y > \tau$. From the definition of conditional probability:
\begin{align}
    g(\rvx) := p(\rvx | y > \tau, D_n) = \frac{\int_{\tau}^\infty p(\rvx, y | D_n) \diff y}{\int_{\tau}^{\infty} p(y  | D_n) \diff y} \neq p(\rvx | y, D_n),
\end{align}
so \autoref{eq:numer} does not hold immediately. Thus, we have illustrated that the above density ratio may not be proportional to expected improvement. 
\end{proof}

\subsection{Statements Regarding Likelihood-free BO}
\label{app:lfbo-proofs}

\variationalintegral*
\begin{proof}
From the definition of $f^\star$, we have that:
\begin{gather*}
    f^\star(f'(s)) = s f'(s) - f(s) \\
    (f^\star \circ f')'(s) = f'(s) + s f''(s) - f'(s) = s f''(s)
\end{gather*}
The derivative of $\bb{E}_{p(y | \rvx)}[u(y, \tau) f'(s)] - f^\star(f'(s))$ is thus:
\begin{align}
    \bb{E}_{p(y | \rvx)}[u(y, \tau)] f''(s) - s f''(s); \label{eq:derivative}
\end{align}
since $f$ is strictly convex, $f''(s) > 0$, and \autoref{eq:derivative} is zero if and only if $s = \bb{E}_{p(y | \rvx)}[u(y, \tau)]$. Furthermore, the derivative of \autoref{eq:derivative} is 
\begin{align}
    \bb{E}_{p(y | \rvx)}[u(y, \tau)] f'''(s) - s f'''(s) - f''(s),
\end{align}
which is negative when $s = \bb{E}_{p(y | \rvx)}[u(y, \tau)]$. Under this value of $s$, the first order derivative is zero and the second order derivative is negative, so this value of $s$ is a unique maximizer. 
\end{proof}

\mainthm*

\begin{proof}
From \autoref{lemma:variational-integral}, we have that for all $\rvx$ queried in $\gD_n$,
$$
     \left(\argmax_{S: \gX \to \R} \bb{E}_{\gD_n}[u(y; \tau) f'(S(\rvx')) - f^\star(f'(S(\rvx'))]\right)(\rvx) = \bb{E}_{p(y | \rvx)}[u(y; \tau)].
$$
As $n \to \infty$, $\gD_n$ converges to a probability distribution $p(\rvx)$ that is fully supported over $\gX$. Let us denote the limiting distribution of $\gD_n$ as $p(\rvx)$, and define $p(\rvy | \rvx)$ as the conditional distribution from the observation process.
Applying the consistency assumption to $\hat{S}_{\gD_n, \tau}$, we have that $\forall \rvx \in \gX, \tau \in \gT$:
\begin{align}
    \plim_{n \to \infty} L^{\mathrm{LF}(u)}(\rvx, \gD_n, \tau) &:= \plim_{n \to \infty} \hat{S}_{\gD_n, \tau}(\rvx) \nonumber \\
    &:= \plim_{n \to \infty} \left(\argmax_{S: \gX \to \R} \bb{E}_{\gD_n}[u(y; \tau) f'(S(\rvx')) - f^\star(f'(S(\rvx'))]\right)(\rvx) \\&\ = \left(\argmax_{S: \gX \to \R} \bb{E}_{p(y | \rvx')p(\rvx')}[u(y; \tau) f'(S(\rvx')) - f^\star(f'(S(\rvx'))]\right)(\rvx) = \bb{E}_{p(y | \rvx)}[u(y; \tau)].
\end{align}
Applying the consistency assumption to $p(y | \rvx, \gD_n)$, we have that:
\begin{align}
     \plim_{n \to \infty} \gL^{(u)}(\rvx, \gD_n, \tau) = \bb{E}_{p(y|\rvx, \gD_n)}[u(y; \tau)] = \bb{E}_{p(y|\rvx)}[u(y; \tau)],
\end{align}
completing the proof.
\end{proof}

In the following, we show that the same expected utility function value can be achieved by different Gaussian distributions. 
\begin{proposition}[Different Gaussian distributions can have the same expected utility]\label{prop:gauss-acqf}
Let $u(y; \tau)$ be non-negative, continuous and non-decreasing in the $y$ argument for a given $\tau$. For any Gaussian $p_1(y) = \gN(\mu_1, \sigma_1^2)$, there exists a Gaussian $p_2(y ) = \gN(\mu_2, \sigma_2^2)$ such that $p_1 \neq p_2$ and 
$$
\bb{E}_{p_1(y)}[u(y; \tau)] = \bb{E}_{p_2(y)}[u(y; \tau)].
$$
\end{proposition}
\begin{proof}
Let $U := \bb{E}_{p_1(y)}[u(y; \tau)]$, and choose any fixed $\sigma_2$ smaller than $\sigma_1$. We have that:
\begin{align}
    \lim_{\mu_2 \to \infty} \bb{E}_{p_2(y )}[u(y; \tau)] &= \max_y u(y; \tau) \\
    \lim_{\mu_2 \to -\infty} \bb{E}_{p_2(y )}[u(y; \tau)] &= \min_y u(y; \tau).
\end{align}
We split $u(y; \tau)$ into two cases:
\begin{itemize}
\item If $\max_y u(y; \tau) = \min_y u(y; \tau)$, then the desired condition is trivially satisfied.
\item Otherwise, $U \in (\min_y u(y; \tau), \max_y u(y; \tau))$; since $ \bb{E}_{p_2(y)}[u(y; \tau)]$ is a continous function over $\mu_2$ when $\sigma_2$ is fixed, then from the intermediate value theorem, there must exist some $\mu_2$ such that $U = \bb{E}_{p_2(y)}[u(y; \tau)]$.
\end{itemize}
Thus, there exists a different Gaussian with the same expected utility. 
\end{proof}

\section{Experimental Details}

\subsection{Synthetic Evaluation of \methodshort{} and BORE (Section \ref{sec:exp:theory})}
\label{app:sec:syn}
Our classifier model is a two layer fully-connected neural network with $128$ units at each layer; we repeat all experiments with $5$ random seeds. For each evaluation, we optimize the model with batch gradient descent for 1000 epochs, using an Adam optimizer with learning rate $0.01$ and weight decay $10^{-6}$. We illustrate the learned acquisition functions for one random seed and compare it with the ground truth ones in \autoref{fig:synthetic-equivalence-appendix}. We placed a non-decreasing transformation over the BORE result to make it directly comparable to PI; suppose the density ratio predictor of BORE is $C_{\mathrm{BORE}}$, we apply the following transformation (and implement a numerically stable version of it):
\begin{align}
    C_{\mathrm{BORE}} \mapsto \frac{C_{\mathrm{BORE}}}{C_{\mathrm{BORE}} + \frac{1 - \gamma}{\gamma}},
\end{align}
where $\gamma = p(y > \tau | \gD_n)$. This does not modify the optimum of BORE, and is only applied when we would compare against the ground truth acquisition functions.

\begin{figure}
    \centering
    \includegraphics[width=0.8\textwidth]{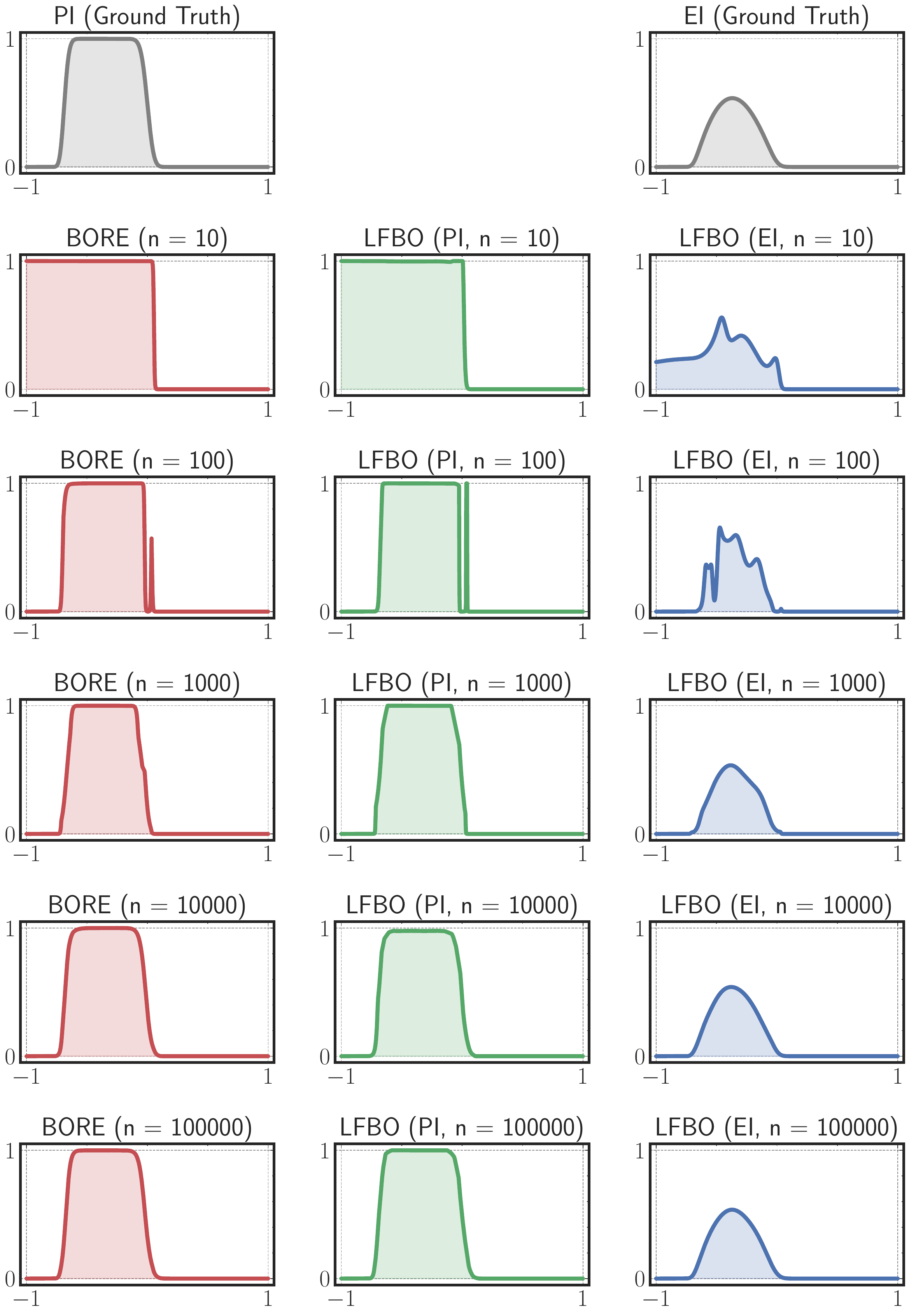}
    \caption{Acquisition functions for the synthetic problem. (\textit{Top row}) ground truth PI and EI values; (\textit{bottom rows}) acquisition functions obtained by BORE (\textit{\textit{left column}}), LFBO (PI, \textit{middle column}), and LFBO (EI, \textit{right column}) with the number of observations $n \in \{10, 10^2, 10^3, 10^4, 10^5, 10^6\}$. The acquisition functions are closer to the ground truth with more observations. }
    \label{fig:synthetic-equivalence-appendix}
\end{figure}

\newpage
\subsection{Hyperparameters for Real-world Benchmarks (Section~\ref{sec:exp:pi-vs-ei})}
\label{app:exp-real}
We provide a description for the hyperparameter search space of the HPOBench \citep{klein2019tabular} and NAS-Bench-201 \citep{dong2020bench} problems in Table~\ref{tab:config-space}.

\begin{table}[htb]
    \centering
    \begin{tabular}{c|c c}
    \toprule
        & Hyperparameter & Choices\\
    \midrule
         \multirow{9}{7em}{HPOBench} & Initial LR & \{0.0005, 0.001, 0.005, 0.01, 0.05, 0.1\} \\
         &  Batch Size & \{8, 16, 32, 64\} \\
         & LR Schedule & \{cosine, fix\} \\
         & Activation of Layer 1 & \{relu, tanh\} \\
         & Activation of Layer 2 & \{relu, tanh\} \\
         & Width of Layer 1 & \{16, 32, 64, 128, 256, 512\} \\
         & Width of Layer 2 & \{16, 32, 64, 128, 256, 512\} \\
         & Dropout rate of Layer 1 & \{0.0, 0.3, 0.6\} \\
         & Dropout rate of Layer 2 & \{0.0, 0.3, 0.6\} \\
    \midrule
    \multirow{6}{7em}{NAS-Bench-201} & Arc-0 Operation & \{zeroize, skip-connect, conv-$1 \times 1$, conv-$3 \times 3$, avg-pool-$3 \times 3$\} \\
    & Arc-1 Operation & \{zeroize, skip-connect, conv-$1 \times 1$, conv-$3 \times 3$, avg-pool-$3 \times 3$\} \\
    & Arc-2 Operation & \{zeroize, skip-connect, conv-$1 \times 1$, conv-$3 \times 3$, avg-pool-$3 \times 3$\} \\
    & Arc-3 Operation & \{zeroize, skip-connect, conv-$1 \times 1$, conv-$3 \times 3$, avg-pool-$3 \times 3$\} \\
    & Arc-4 Operation & \{zeroize, skip-connect, conv-$1 \times 1$, conv-$3 \times 3$, avg-pool-$3 \times 3$\} \\
    & Arc-5 Operation & \{zeroize, skip-connect, conv-$1 \times 1$, conv-$3 \times 3$, avg-pool-$3 \times 3$\} \\
    \bottomrule
    \end{tabular}
    \caption{Configuration spaces for HPOBench and NAS-Bench-201 problems.}
    \label{tab:config-space}
\end{table}

We follow \citep{tiao2021bore} to use Multi-layer perceptions (MLP), Random Forest (RF, \citet{breiman2001random}), and gradient boosted trees (XGBoost, \citet{chen2016xgboost}) to implement the probabilistic classifier. Detailed hyperparameters are summarized below:

\textbf{MLP:} \url{https://keras.io/}
\vspace{-3mm}
\begin{itemize}[noitemsep]
    \item Number of hidden layers: $2$
    \item Number of hidden units: $32$
    \item Activation function: \texttt{ReLU}
    \item Optimizer: \texttt{Adam} (with default parameters \href{https://www.tensorflow.org/api_docs/python/tf/keras/optimizers/Adam}{here})
    \item Batch size: $64$
\end{itemize}

\textbf{Random Forest:} \url{https://scikit-learn.org/}
\vspace{-3mm}
\begin{itemize}[noitemsep]
    \item Number of trees in the forest: $1000$
    \item Minimum number of samples required to split an internal node: $2$
    \item Maximum depth of the tree: \texttt{None}, nodes are expanded  until all leaves contain less than 2 samples.
    \item Minimum number of samples required to be at a leaf node: $1$
    \item More parameter values can be found \href{https://scikit-learn.org/stable/modules/generated/sklearn.ensemble.RandomForestClassifier.html}{here}.
\end{itemize}

\textbf{XGBoost:} \url{https://xgboost.readthedocs.io/}
\vspace{-3mm}
\begin{itemize}[noitemsep]
    \item Number of boosting rounds: $100$
    \item Minimum sum of instance weight (hessian) needed in a child: $1$
    \item Boosting learning rate: $0.3$
    \item More parameter values can be found \href{https://xgboost.readthedocs.io/en/stable/python/python_api.html#xgboost.XGBClassifier}{here}.
\end{itemize}

\subsection{Ablation Study over the Threshold}
\label{app:sec:ablation-gamma}
\begin{figure}[htb]
    \centering
    \includegraphics[width=0.7\textwidth]{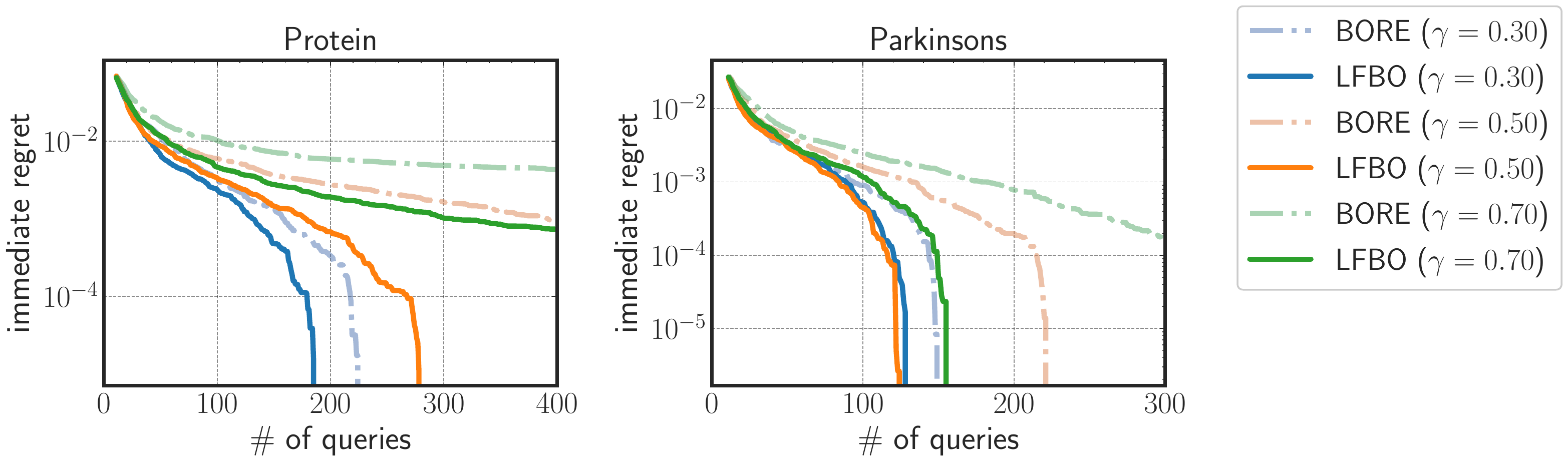}
    \caption{Results on two HPOBench problems with different thresholds. LFBO uses the utility function for EI.%
    }
    \label{fig:gamma}
\end{figure}

Both BORE and LFBO (EI) include a hyperparameter $\gamma$, which includes less positive samples when smaller and may affect optimization performance empirically. We use the XGBoost classifier on two HPOBench problems with different values of $\gamma$ and plot the results in \autoref{fig:gamma}. These results show that while we observe a performance drop with a larger $\gamma$ in both cases, LFBO (EI) is less sensitive to the choice of $\gamma$, which is reasonable since it take into account how much the observed value exceeds the threshold. 

\subsection{\methodshort{} for Composite Functions}
\label{app:sec:composite}

\paragraph{\methodshort{} model and objectives.} Our composite neural network is $C_\theta(\rvx): \gX \to \R$ parametrized by $\theta$: $$C_\theta(\rvx) = \frac{u(-\norm{h_\theta(\rvx) - z_\star}_2^2; \tau)}{u(-\norm{h_\theta(\rvx) - z_\star}_2^2; \tau) + 1},$$ where $h_\theta(\rvx): \gX \to \R^d$ is a neural network with the training parameters $\theta$. Essentially, $C_\theta$ is a neural network that contains a special final layer that incorporates the structure of the objective function (see \autoref{fig:comp-nn}). 
We also considered replacing the final layer with another neural network, but observed worse performance; this is not hard to expect since we would require some data to learn a suitable mapping in this case. 

We consider training $C_\theta(\rvx)$ with our \methodshort{} objective plus a regularization term that encourages $h_\theta(\rvx)$ (the neural network) to produce similar outputs to $h(\rvx)$ (the black-box); for EI, this leads to the following objective function:
\begin{align}
    \frac{|\gD_n^+|}{|\gD_n|} \bb{E}_{\gD_n^{+}}[(-\norm{\rvy - \rvz_\star}_2^2 - \tau) \log C_\theta(\rvx)] + \bb{E}_{\gD_n}[\log (1 - C_\theta(\rvx))] + \bb{E}_{(\rvx, \rvy) \sim \gD_n}[-\norm{h_\theta(\rvx) - \rvy}_2^2]
\end{align}
where $\gD_n = \{(\rvx, \rvy)\}_{i=1}^{n}$ contains observations from the black-box $h(\rvx)$, and $\gD_n^+$ is the subset of $\gD_n$ where $-\norm{\rvy - \rvz_\star}_2^2 > \tau$.  We also tried different weights between the \methodshort{} objective and the regularization term, which gave similar results. 

\begin{figure}
    \centering
    \includegraphics[width=0.8\textwidth]{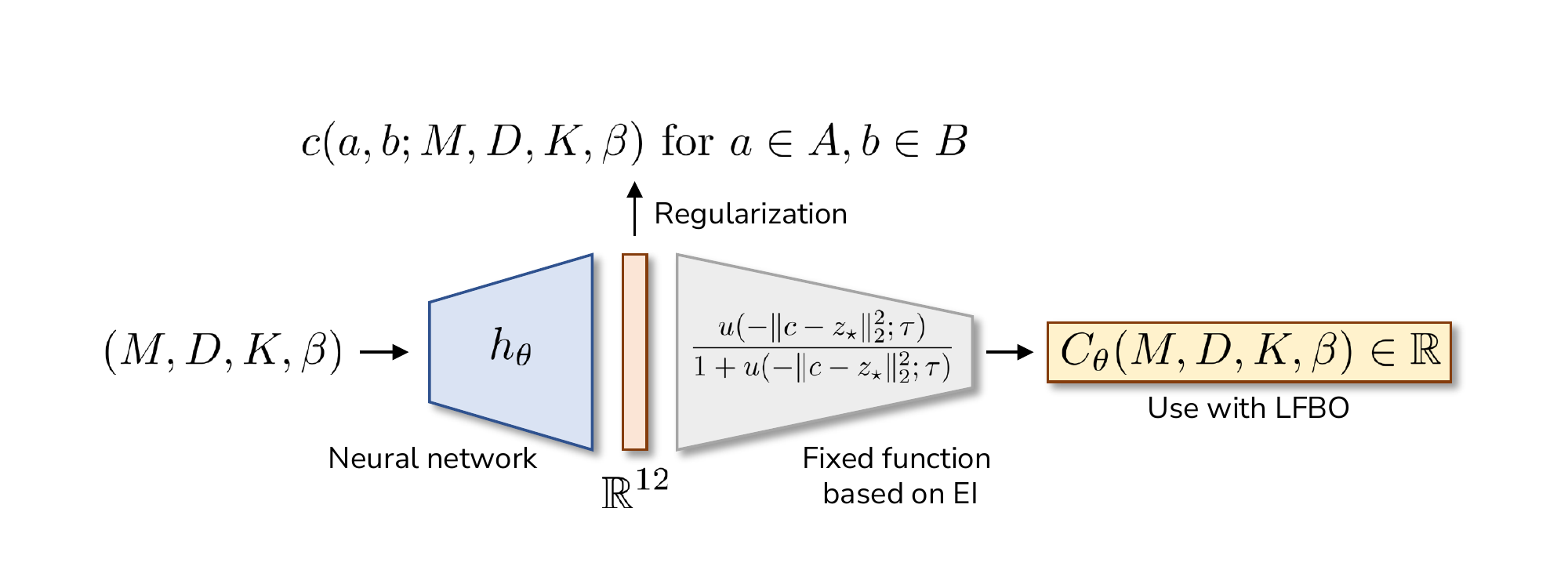}
    \caption{Structure for the classifier model in the composite case. We regularize the intermediate layer such that it is close to the real black-box outputs from $h$.}
    \label{fig:comp-nn}
\end{figure}

\paragraph{Test problem.} We evaluate our methods on the environmental model function~\citep{bliznyuk2008bayesian}, which models the spilling of pollutants at two sites into a long and narrow holding channel in a chemical accident. The model is:
\begin{align}
    c(a, b; M, D, K, \xi) = \frac{M}{\sqrt{4 \pi D b}} \exp\left(\frac{-a^2}{4 D b}\right) + \frac{\bb{I}(b - \xi > 0) M}{\sqrt{4 \pi D (b - b_0)}} \exp\left(\frac{- (a - K)^2}{4 D (b - \xi)}\right)
\end{align}
where $M$ is the mass of pollutant spilled at each location, $D$ is the diffusion rate in the channel, $K$ is the location of the second spill, and $\xi$ is the time of the second spill. The objective function is:
\begin{align}
    \sum_{(a, b) \in A \times B} (c(a, b; M, D, K, \xi) - c(a, b; M_0, D_0, K_0; \xi_0))^2
\end{align}
where $A = \{0, 1, 2.5\}$, $B = \{15, 30, 45, 60\}$ are grid locations, and $(M_0, D_0, K_0, \xi_0)$ are the underlying true values of these parameters. This function is a grey-box function in the sense that $c(a, b; M, D, K, \xi)$ for each $(a, b) \in A \times B$ is a black-box function. 

\paragraph{Implementation and hyperparameters.}
We use our own implementation for GP (EI), \methodshort{} (EI) and composite \methodshort{} (EI), and use the official implementation\footnote{\url{https://github.com/RaulAstudillo06/BOCF/blob/master/test_4b.py}} (with minor modifications for bug fixes) for composite GP (EI). We note that the regrets for composite GP (EI) appear larger than the ones reported in \citep{astudillo2019bayesian} because of the evaluated statistics. \citet{astudillo2019bayesian} use average of log10 regret, whereas we use log10 of average regret; since log10 is concave, our statistics will be larger than theirs (which would favor higher variance across random seeds). 

For all the methods, we first take 10 random samples and then proceed with the Bayesian optimization algorithm; \autoref{fig:comp-regret} plots the immediate regret \textit{after} the initial random samples have been collected.
For composite \methodshort{} (EI), we select $\tau$ to be the $10$-th percentile of the existing observed $g(\rvx)$ values. Our $h_\theta(\rvx)$ is a two layer fully connected neural network with $64$ units in each layer; we find that the results are insensitive to slight changes to the neural network architecture.

\section{Additional Experimental Results}
\subsection{LFBO on Additional Continuous Functions}

\begin{figure}
    \centering
    \includegraphics[width=0.7\textwidth]{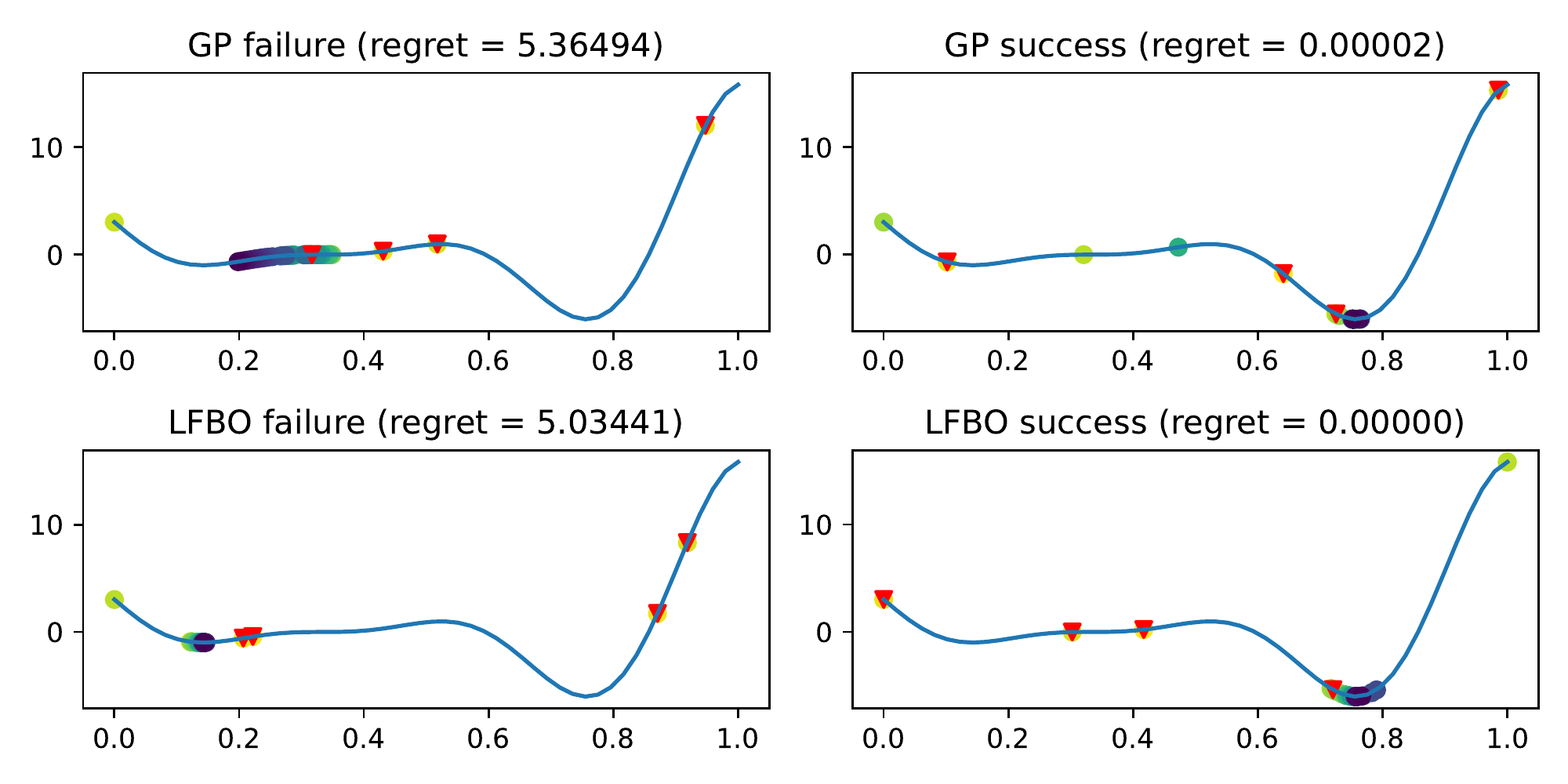}
    \caption{Success and failure cases for LFBO and GP (with EI) on minimizing the Forrester function~\citep{sobester2008engineering}. The red triangle points indicate the initial random observations, and the subsequent observations are denoted by color circles (from bright to dark).
    }
    \label{fig:forrester}
\end{figure}

Apart from the composite function, we further evaluate LFBO on some additional continuous functions, such as the Forrester function (example from in \autoref{fig:forrester}; we observe both successful and unsuccessful optimization results for both LFBO and GP, which depends on the initial random queries. Failures are mostly due to initial samples that causes the model to be overconfident about the global valley between $[0.6, 0.8]$, and even GPs cannot avoid them~\citep{deshpande2021calibration} %
(these occur roughly 25\% of the time for both methods). Improving LFBO on these failure cases is an interesting future work.

\subsection{Runtime Comparisons between LFBO and GP-based Methods}

We perform these comparisons on the same machine. 
On NAS-Bench-201's CIFAR-10 dataset, the average time for LFBO to finish 200 steps of BO is around 110 seconds,
while the average time for GP-EI (implemented with the BoTorch library) to finish 200 steps of BO is around 600 seconds. 
Thus LFBO is $5 \times$ faster than the GP in this setting. We also find that model fitting and acquisition function optimization takes roughly the same time with LFBO, but the latter was much slower with GPs. Although both methods require some training time given new observations, the inference time of LFBO is much faster, 
since basic GP requires $\mathcal{O}(n^3)$ computational complexity to perform posterior inference, while in LFBO the model size (and by extension inference time) is constant.
While sparse GPs can reduce inference time to the order of $\mathcal{O}(n)$, they can still be more expensive than LFBO (as $n$ becomes very large).

On Bliznyuk, composite EI with GP (using the official implementation based on GPyTorch) is much slower (where 40 iterations took around 1200s), around $30\times$ slower than LFBO (which took around 40s). Even with a more efficient implementation using multi-task GPs~\citep{maddox2021bayesian}\footnote{\url{https://botorch.org/tutorials/composite_bo_with_hogp}}, the 40 iterations took around 640s on the same CPU, which was around $15\times$ slower. 
We believe this is probably because the model has to take additional Monte Carlo samples as composite EI is no longer analytically tractable. While this gap could be smaller if further optimizations over the GP implementations are made, we believe that this demonstrates the efficiency of LFBO methods even with additional training.

\end{document}